\documentclass[11pt]{article}

\usepackage[final]{acl}

\usepackage{times}
\usepackage{latexsym}
\usepackage[T1]{fontenc}
\usepackage[utf8]{inputenc}
\usepackage{microtype}
\usepackage{inconsolata}

\usepackage{graphicx}
\usepackage{subcaption}
\usepackage{booktabs}
\usepackage{enumitem}
\usepackage{listings}
\usepackage{amsmath}
\usepackage{amssymb}
\usepackage{mathtools}
\usepackage{amsthm}
\usepackage{multirow}
\usepackage{bm}
\usepackage{tikz}
\usepackage{colortbl}
\usepackage{algorithm}
\usepackage{algorithmic}
\usepackage[capitalize,noabbrev]{cleveref}
\usepackage[disable,textsize=tiny]{todonotes}

\newcommand*\rot{\rotatebox{90}}

\usetikzlibrary{positioning, fit, backgrounds, arrows.meta, shapes.geometric, calc, shadows}

\definecolor{primaryBlue}{HTML}{4A90D9}
\definecolor{lightBlue}{HTML}{E8F1FA}
\definecolor{primaryGreen}{HTML}{5CB85C}
\definecolor{lightGreen}{HTML}{E8F5E8}
\definecolor{accentOrange}{HTML}{F5A623}
\definecolor{lightOrange}{HTML}{FEF5E7}
\definecolor{accentRed}{HTML}{E74C3C}
\definecolor{lightRed}{HTML}{FDEDEC}
\definecolor{darkGray}{HTML}{4A4A4A}
\definecolor{medGray}{HTML}{9B9B9B}
\definecolor{lightGray}{HTML}{F5F5F5}

\theoremstyle{plain}
\newtheorem{theorem}{Theorem}[section]
\newtheorem{proposition}[theorem]{Proposition}

\theoremstyle{definition}

\theoremstyle{remark}

\title{Representation-based Masked Diffusion Model}

\author{
  Yangrong Hu\textsuperscript{1},
  Ding Huang\textsuperscript{2,\textdaggerdbl},
  Xueyu Zhou\textsuperscript{1},
  \and Jian Huang\textsuperscript{1,2,\textdagger} \\
  \textsuperscript{1}Department of Data Science and Artificial Intelligence \\
  \textsuperscript{2}Department of Applied Mathematics \\
  The Hong Kong Polytechnic University, Hong Kong SAR, China \\
  \small\textit{\{yangrong.hu,ding.huang,xueyu.zhou\}@connect.polyu.hk} \\
  \small\textit{j.huang@polyu.edu.hk}
}

\begin{document}
\maketitle
\begingroup
\renewcommand{\thefootnote}{\textdagger}
\footnotetext{\raggedright Corresponding author: Jian Huang.}
\endgroup
\begingroup
\renewcommand{\thefootnote}{\textdaggerdbl}
\footnotetext{\raggedright Present address: ByteDance Seed, Beijing, China. Email: \texttt{huanhgding@bytedance.com}.}
\endgroup

\begin{abstract}

    Masked Diffusion Models (MDMs) have emerged as a compelling paradigm for language modeling, offering the capability for efficient parallel text generation. 
    However, existing parallel sampling methods typically update multiple masked tokens independently and ignore the complex mutual dependencies among the masked tokens. 
    This independent updating mechanism lacks global coordination and might lead to incoherent outputs. To address this limitation, we propose Representation-based Masked Diffusion Model (RMDM),
    a framework that leverages the text representation to explicitly encode global semantics and help to parallel update tokens more precisely. 
    Specifically, we first encode text into a continuous semantic space using a pretrained encoder and learn an invertible transformation that normalizes the representation distribution to a Gaussian prior, facilitating efficient sampling during generation. 
    Conditioned on this latent semantic representation, we train a masked diffusion model to learn the conditional text distribution, where the representation serves as global semantic guidance to coordinate parallel token updates and faithfully approximate the target distribution.
    Empirical results demonstrate that RMDM significantly improves generation quality, particularly in aggressive few-step sampling regimes.
\end{abstract}

\begin{figure*}[!t]
    \centering
    \resizebox{0.8\linewidth}{!}{
    \begin{tikzpicture}[
        font=\sffamily,
        node distance=1.0cm and 1.0cm,
        base/.style={
            align=center,
            line width=0.8pt,
            font=\small
        },
        tensor/.style={
            base,
            draw=primaryBlue!80, 
            fill=lightBlue, 
            rounded corners=3pt,
            minimum height=0.8cm, 
            minimum width=1.2cm
        },
        process/.style={
            base,
            draw=darkGray, 
            fill=white, 
            rounded corners=3pt, 
            minimum height=0.8cm, 
            minimum width=1.6cm,
            drop shadow={opacity=0.1}
        },
        latent/.style={
            base,
            circle, 
            draw=accentRed!80, 
            fill=lightRed, 
            minimum size=1.0cm, 
            inner sep=1pt
        },
        encoder/.style={
            base,
            rounded corners=3pt,
            draw=accentOrange!80, 
            fill=lightOrange, 
            minimum height=0.8cm,
            minimum width=1.2cm
        },
        model/.style={
            base,
            rounded corners=4pt,
            minimum height=1.4cm, 
            minimum width=2.8cm,
            draw=primaryGreen!80, 
            fill=lightGreen
        },
        arrow/.style={-{Stealth[length=6pt, width=4pt]}, line width=0.8pt, color=darkGray},
        cond_arrow/.style={-{Stealth[length=6pt, width=4pt]}, line width=1.0pt, color=accentRed, dashed},
        label/.style={font=\scriptsize\bfseries, color=darkGray}
    ]

    \node[encoder] (enc_shape) at (0,0) {Encoder\\$g_\psi$};
    
    \node[tensor, right=1.2cm of enc_shape, fill=lightOrange, draw=accentOrange!80] (z1) {$\bm{z}_1$};
    
    \node[process, right=1.2cm of z1, fill=lightBlue, draw=primaryBlue!80, minimum width=2.2cm] (meanflow) {MeanFlow\\$u_\phi$};
    
    \node[latent, right=1.2cm of meanflow] (z0) {$\bm{z}_0$};
    \node[right=0.2cm of z0, font=\scriptsize, color=medGray] (gaussian) {$\sim\mathcal{N}(\bm{0},\bm{I})$};

    \draw[arrow] (enc_shape) -- (z1);
    \draw[arrow] (z1) -- (meanflow);
    \draw[arrow] (meanflow) -- (z0);

    
    \node[tensor, below=0.8cm of enc_shape] (x0) {$\bm{x}_0$};
    
    \node[process, right=1.2cm of x0, fill=lightGray] (mask) {Mask};
    
    \node[tensor, right=1.4cm of mask] (xt) {$\bm{x}_t$};
    
    \node[model, anchor=center] (mdlm) at (z0 |- x0) {\textbf{RMDM}\\$p_\theta(\bm{x}_0 | \bm{x}_t, \bm{z}_0)$};
    
    \node[tensor, right=0.6cm of mdlm, fill=lightGreen, draw=primaryGreen!80] (x_hat) {$\hat{\bm{x}}_0$};

    \draw[arrow] (x0) -- (mask);
    \draw[arrow] (mask) -- (xt);
    \draw[arrow] (xt) -- (mdlm);
    \draw[arrow] (mdlm) -- (x_hat);
    
    \draw[arrow] (x0) -- (enc_shape); 
    
    \draw[cond_arrow] (z0) -- node[midway, right, font=\scriptsize, color=accentRed] {} (mdlm);

    \begin{scope}[on background layer]
        \node[fit=(enc_shape)(z1)(meanflow)(z0)(gaussian), rounded corners=5pt, fill=lightOrange!20, draw=accentOrange!30, dashed] (stage1) {};
        \node[above, font=\scriptsize\bfseries, color=accentOrange!80] (stage1_label) at (stage1.north) {Stage I: Representation Mapping};
        
        \node[fit=(x0)(mask)(xt)(mdlm)(x_hat), rounded corners=5pt, fill=lightGreen!20, draw=primaryGreen!30, dashed] (stage2) {};
        \node[below, font=\scriptsize\bfseries, color=primaryGreen!80] at (stage1_label |- stage2.south) {Stage II: Generative Modeling};
    \end{scope}

    \end{tikzpicture}%
    }
    \caption{
        Stage I aligns the latent space $\bm{z}_0$, derived from a pretrained encoder $g_{\psi}(\bm{x}_0)$, with a Gaussian prior via MeanFlow $u_\phi$. Stage II trains the RMDM $p_\theta$ to reconstruct original data $\bm{x}_0$ from corrupted data $\bm{x}_t$ guided by the latent representation $\bm{z}_0$.}
    \label{fig:framework}
\end{figure*}
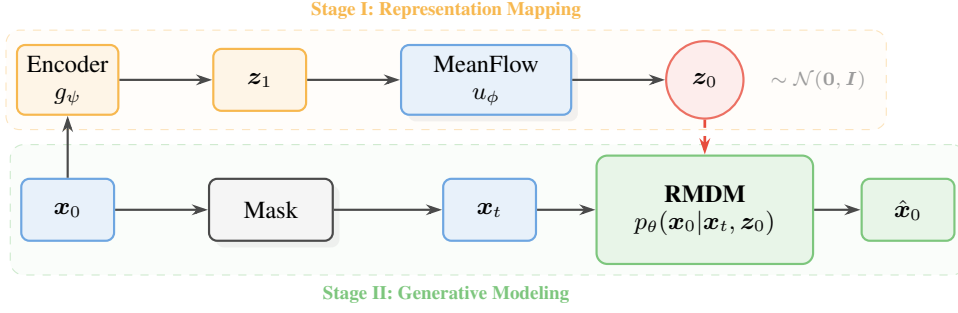
\section{Introduction}

Large Language Models (LLMs) have achieved huge success for modeling text data ~\citep{brown2020language, gpt2, gpt4, deepseekv3}, primarily through the autoregressive (AR) paradigm.
While AR models excel at capturing complex dependencies by generating tokens sequentially, their inference speed is fundamentally constrained by their $O(L)$ sequential complexity for a sequence of length $L$. This latency bottleneck has spurred significant interest in non-autoregressive or parallel generation frameworks~\citep{gu2018non}.
Masked Diffusion Models (MDMs) ~\citep{austin2021structured, sahoo2024mdlm, lou2024discrete, ou2025radd} have recently emerged as a compelling alternative that enables parallel token generation via iterative denoising.
Compared to continuous diffusion models, MDMs utilize a \emph{discrete diffusion} process, where the forward process progressively adds noise to a clean sequence by replacing its subset with special \texttt{[MASK]} tokens instead of gaussian noise and yields a partially observed context. 
A denoiser is then trained to invert this corruption by predicting the original tokens from the context.
In the backward process, MDMs utilize the denoiser to generate the whole sequence by repeatedly refining a sentence that consists entirely of masked tokens.

In practical deployment, MDMs are required to efficiently update multiple masked tokens in a single iteration. 
Most parallel methods~\citep{llada,llada1.5,dream7b,chang2022maskgit,fastdllm} sample each masked position independently conditioned on the currently observed context, while ignoring the dependency between masked tokens.
Such independent sampling approximates the joint distribution over the masked positions with a product of token-wise marginal distributions.
However, the factorized approximation is generally misaligned with natural language, where masked tokens often exhibit strong residual dependencies induced by syntactic agreement, long-range constraints, multi-word expressions, and entity consistency.
As a result, aggressive parallelization can produce sequences that are locally plausible at each position yet globally inconsistent under the true joint distribution. We formalize and quantify this \textbf{Conditional Dependency Gap} in Section~\ref{sec:method:gap}, and build our method around mitigating it without sacrificing parallelism and efficiency.
\paragraph{Our approach.}
To mitigate this inconsistency, we introduce a global latent representation $z$ that encodes the semantic information of the whole sequence.
This representation captures shared global factors, such as topic and intent, which largely explain the remaining dependence among masked tokens.
Therefore, conditioning on $z$  reduces residual cross-token dependence and makes parallel prediction closer to the true joint conditional distribution.
We formalize this intuition as a reduction of the \emph{Conditional Dependency Gap} in Section~\ref{sec:method:gap}.
To make the latent variable $z$ operational, we use pretrained encoders, such as BERT models~\citep{bert,modernbert} and modern large-scale embedding models~\citep{qwen3embedding}, to extract contextual representations of the input text.
The representation encodes global attributes of the underlying sequence and captures high-level semantic factors that are difficult to infer from a partially revealed context.
By providing this missing global information, the representation serves as a natural candidate for $z$ and helps reduce the Conditional Dependency Gap during parallel updates.
A key complication is that these representations are readily available from the ground-truth text during training but are unobserved at inference time.
Directly conditioning on such embeddings would therefore introduce a train--test mismatch unless $z$ can be sampled from a well-defined prior during generation.
To resolve the training--inference mismatch, we employ MeanFlow~\citep{shi2024meanflow} to map the empirical distribution of pretrained sequence representations to a Gaussian prior. This transport allows us to directly sample a global latent $z \sim \mathcal{N}(0,I)$ at generation time, which serves as a global condition for the RMDM to ensure coherent parallel updates.
Although variational inference~\citep{VAE, xie2025variationalautoencodingdiscretediffusion} is a classical approach to Gaussian latents, its reconstruction and KL objectives often conflict in high-capacity language models, leading to unstable optimization or posterior collapse. Approaches such as CCDD~\citep{zhou2025coevolutionarycontinuousdiscretediffusion} co-evolve continuous and discrete diffusions, whereas RMDM samples a fixed latent once and reuses it throughout discrete denoising. The pretrained encoder and MeanFlow network are used to construct training latents but are not invoked during generation.
Experimental results demonstrate that RMDM consistently outperforms standard MDMs in sample quality for a fixed number of sampling steps and achieves approximately $3.6\times$ higher measured throughput at matched generation quality in our A800 evaluation.
We provide a more detailed discussion of related discrete diffusion language models, continuous diffusion language models, and latent-augmented masked diffusion methods in Appendix~\ref{app:related}.

\section{Preliminaries}

\subsection{Continuous Diffusion and MeanFlow}
\label{sec:pre:continuous}

Continuous diffusion models ~\citep{ho2020denoising, song2020score} are built on the principle of a forward process that gradually corrupts data $\bm{z}_0 \in \mathbb{R}^d$ into Gaussian noise $\bm{z}_1 \sim \mathcal{N}(\bm{0}, \bm{I})$ over a continuous time $t \in [0, 1]$. The generative process is then defined by learning to reverse this corruption—essentially "denoising" the latent—by following a velocity field that guides the noise back to the data manifold. Flow matching ~\citep{lipman2022flow} simplifies this by training a model $\bm{v}_\theta$ to match the conditional velocity field:
\begin{align*}    
    \mathcal{L}_{\text{CFM}}(\theta) = \mathbb{E}_{t, \bm{z}_0, \bm{z}_1} \left[ \| \bm{v}_\theta(\bm{z}_t, t) - (\bm{z}_1 - \bm{z}_0) \|^2 \right],
\end{align*}
where $\bm{z}_t = (1-t)\bm{z}_0 + t\bm{z}_1$ represents the linear interpolation between data and noise. While standard flow matching requires numerical ODE integration over multiple steps, MeanFlow~\citep{shi2024meanflow} enables efficient one-step generation by modeling the \textit{average velocity} $u(\bm{z}_t, r, t) \triangleq \frac{1}{t-r} \int_r^t \bm{v}(\bm{z}_\tau, \tau) d\tau$. The model $u_\theta$ is trained via the MeanFlow identity:
\begin{align*}
    \resizebox{.95\linewidth}{!}{$
    \mathcal{L}_{\text{MF}}(\theta) = \mathbb{E} \left\| u_\theta(\bm{z}_t, r, t) - \text{sg} \left[ \bm{v}_t - (t-r) \frac{d}{dt} u_\theta(\bm{z}_t, r, t) \right] \right\|^2
    $}
\end{align*}
Where $sg$ is the stop gradient operator.
Once optimized, $\bm{z}_0$ can be recovered from noise $\bm{z}_1$ in a single step as $\bm{z}_0 = \bm{z}_1 - u_\theta(\bm{z}_1, 0, 1)$. 

\subsection{Masked Diffusion Models}
\label{sec:pre:mdm}

While continuous diffusion is well-suited for continuous distributions, language modeling requires a principled approach to handle discrete categorical variables. Masked diffusion models (MDMs) ~\citep{austin2021structured, lou2024discrete} extend the diffusion framework to discrete state spaces by defining a corruption process that gradually replaces tokens $\bm{x}_0$ with an absorbing \texttt{[MASK]} while a sequence with full \texttt{[MASK]} is denoted as state $\bm{m}$.

Similar to the continuous case, the forward process $q(\bm{x}_t | \bm{x}_0) = \text{Cat}(\bm{x}_t; \alpha_t \bm{x}_0 + (1 - \alpha_t) \bm{m})$ interpolates between clean data and the masked prior, where $\alpha_t \in [0, 1]$ is a decreasing noise schedule. The generative model $p_\theta(\bm{x}_s | \bm{x}_t)$ learns to invert this masking process by approximating the closed-form posterior:
\begin{align*}
    \resizebox{.99\linewidth}{!}{$
    q(\bm{x}_s | \bm{x}_t, \bm{x}_0) = 
    \begin{cases} 
        \delta_{\bm{x}_t} & \text{if } \bm{x}_t \neq \bm{m}, \\
        \text{Cat}\left(\bm{x}_s; \frac{\alpha_s - \alpha_t}{1 - \alpha_t} \bm{x}_0 + \frac{1 - \alpha_s}{1 - \alpha_t} \bm{m}\right) & \text{if } \bm{x}_t = \bm{m}
    \end{cases}
    $}
\end{align*}
By training a denoising model $p_\theta(\bm{x}_0 | \bm{x}_t)$, MDMs can be optimized via an objective derived from the data log-likelihood $\log p_\theta(\bm{x}_0)$:
\begin{align}
    \mathcal{L}(\theta)
    &=
    \int_0^1 \frac{1}{t}
    \mathbb{E}_{q(\bm{x}_t|\bm{x}_0)}
    \left[
    \sum_{i\in \mathcal{I}_t}
    -\log p_\theta(\bm{x}_0^i|\bm{x}_t)
    \right]
    dt
\label{eq:mdm_elbo}
\end{align}
where $\mathcal{I}_t=\{i:\bm{x}_t^i=\texttt{[MASK]}\}$.
This loss allows the model to predict the categorical distribution of original tokens at each masked position, providing a discrete analogue to the continuous denoising flow described above.

\section{Methodology}
\label{sec:method}


\subsection{The Conditional Dependency Gap}
\label{sec:method:gap}
In this section, we denote $\bm{x}_M$ as the masked data, and $\bm{x}_U$ as the unmasked data in time $t$ of the diffusion process.
A fundamental challenge in Masked Diffusion Models (MDMs) arises from the discrepancy between their training objective and the parallel sampling strategy required for efficiency. While the standard objective~Equation~\ref{eq:mdm_elbo} optimizes the reconstruction of individual tokens $p_\theta(x_i \mid \bm{x}_U)$ independently, fast generation typically necessitates sampling all masked tokens $\bm{x}_M$ simultaneously. Existing parallel samplers implicitly approximate the joint posterior via factorization:
\begin{align*}
    & p_\theta(\bm{x}_M \mid \bm{x}_U) 
    \triangleq  \prod_{i \in M} p_\theta(x_i \mid \bm{x}_U) \\
    \approx & \prod_{i \in M} p_{\text{data}}(x_i \mid \bm{x}_U) \neq p_{\text{data}}(\bm{x}_M \mid \bm{x}_U),
\end{align*}
where $M$ and $U$ denote the sets of masked and unmasked indices, respectively. This factorization assumes that masked tokens are conditionally independent given the context $\bm{x}_U$. However, for structured data, this assumption is often violated, leading to a \textbf{Conditional Dependency Gap}. We quantify this gap using the conditional total correlation 
\begin{figure}[!t]
    \centering
    \begin{minipage}[c]{0.35\linewidth}
        \caption{\textbf{Causal graph.} The global latent $\bm{z}$ acts as a common cause for masked tokens. Conditioning on $\bm{z}$ explains away the correlation between $x_M^1$ and $x_M^2$.}
        \label{fig:dependency_graph}
    \end{minipage}
    \hfill
    \begin{minipage}[c]{0.6\linewidth}
        \centering
        \begin{tikzpicture}[
            node distance=1.0cm and 0.8cm,
            mynode/.style={draw, circle, minimum size=0.85cm, inner sep=0pt, align=center, thick},
            arrow/.style={-stealth, thick}
        ]
            \node[mynode] (z) {$\bm{z}$};
            \node[mynode, below=of z] (xu) {$\bm{x}_U$};
            \node[mynode, left=of xu] (x1) {$x_M^1$};
            \node[mynode, right=of xu] (x2) {$x_M^2$};

            \draw[arrow] (z) -- (x1);
            \draw[arrow] (z) -- (xu);
            \draw[arrow] (z) -- (x2);

            \draw[arrow, dashed] (xu) -- (x1);
            \draw[arrow, dashed] (xu) -- (x2);
        \end{tikzpicture}
    \end{minipage}%
\end{figure}
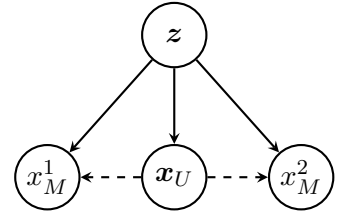
\begin{align*}
    \resizebox{.99\linewidth}{!}{$
    \mathcal{T}(\bm{x}_M \mid \bm{x}_U)
    \triangleq
    D_{\mathrm{KL}}\!\left(
    p_{\text{data}}(\bm{x}_M \mid \bm{x}_U)
    \,\middle\|\,
    \prod_{i\in M} p_{\text{data}}(x_i \mid \bm{x}_U)
    \right)
    $}
\end{align*}
which measures the intrinsic dependence among masked tokens that is not explained by the observed context.

To bridge this gap, we introduce a latent variable $\bm{z}$ that encodes global information about $\bm{x}$, thereby capturing shared factors that induce dependencies among the masked tokens.
\begin{align*}
    \resizebox{.99\linewidth}{!}{$
    \mathcal{T}(\bm{x}_M \mid \bm{x}_U, \bm{z}) = D_{\mathrm{KL}}\!\left(p_{\text{data}}(\bm{x}_M \mid \bm{x}_U, \bm{z}) \,\middle\|\, \prod_{i \in M} p_{\text{data}}(\bm{x}_i \mid \bm{x}_U, \bm{z})\right)
    $}
\end{align*}
Conditioning on $\bm{z}$ is able to reduce the independence gap, as formalized below.


\begin{proposition}[Relation between Conditional Independence and Conditional Dependency Gap]
    \label{prop:dependency_reduction}
    The residual dependency gap $\mathcal{T}(\bm{x}_M \mid \bm{x}_U, \bm{z})$ vanishes if and only if the masked tokens $\bm{x}_M$ are conditionally independent given the context $\bm{x}_U$ and the latent variable $\bm{z}$:
    \begin{align*}
        &\mathcal{T}(\bm{x}_M \mid \bm{x}_U, \bm{z}) = 0 \\
        \iff & x_{m_1} \perp\!\!\!\perp x_{m_2} \perp\!\!\!\perp \dots \perp\!\!\!\perp x_{m_{|M|}} \mid (\bm{x}_U, \bm{z})
    \end{align*}
    Furthermore, the introduction of $\bm{z}$ changes the dependency gap by:
    \begin{align*}
        \Delta \mathcal{T} &\triangleq \mathcal{T}(\bm{x}_M \mid \bm{x}_U) - \mathcal{T}(\bm{x}_M \mid \bm{x}_U, \bm{z}) \\
        &= \sum_{i \in M} I(x_i; \bm{z} \mid \bm{x}_U) - I(\bm{x}_M; \bm{z} \mid \bm{x}_U) 
    \end{align*}
    \end{proposition}
    We refer $\mathcal{T}(\bm{x}_M \mid \bm{x}_U, \bm{z})$ as the residual gap.
    When $\mathcal{T}(\bm{x}_M \mid \bm{x}_U, \bm{z}) = 0$, the conditional joint distribution factorizes, so sampling 
    each $x_i \sim p(x_i \mid \bm{x}_U, \bm{z})$ 
    independently is equivalent to sampling the entire block $\bm{x}_M \sim p(\bm{x}_M \mid \bm{x}_U, \bm{z})$, enabling exact parallel sampling. 
    The proof is provided in Appendix~\ref{app:proof_prop1}.
    Figure~\ref{fig:dependency_graph} shows an example of the casual graph for $x_M$, $\bm{x}_U$ and $\bm{z}$.  More detailed numerical calculation examples are shown in Appendix~\ref{app:example_calculation}.


\subsection{Representation-based Masked Diffusion Model}
\label{sec:method:meanflow}

\begin{figure}[t]
    \centering
    \captionsetup[subfigure]{labelformat=empty}
    \begin{subfigure}[b]{0.22\textwidth}
        \centering
        \includegraphics[width=\linewidth]{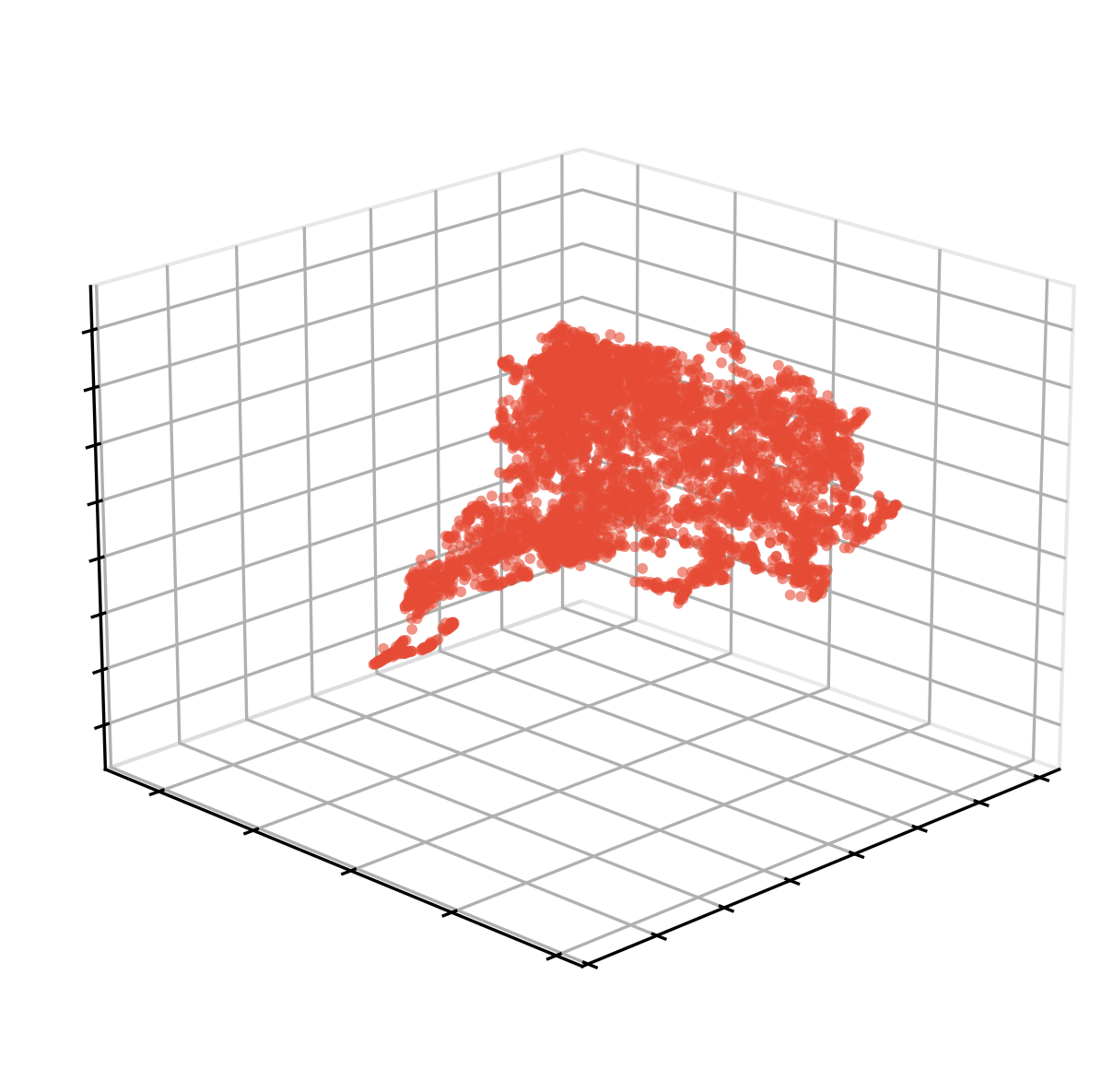}
        \vspace{-0.8cm}
        \caption{\tiny{Contextualized embeddings $\bm{z}_1$}}
        \label{fig:umap_e}
    \end{subfigure}
    \begin{subfigure}[b]{0.22\textwidth}
        \centering
        \includegraphics[width=\linewidth]{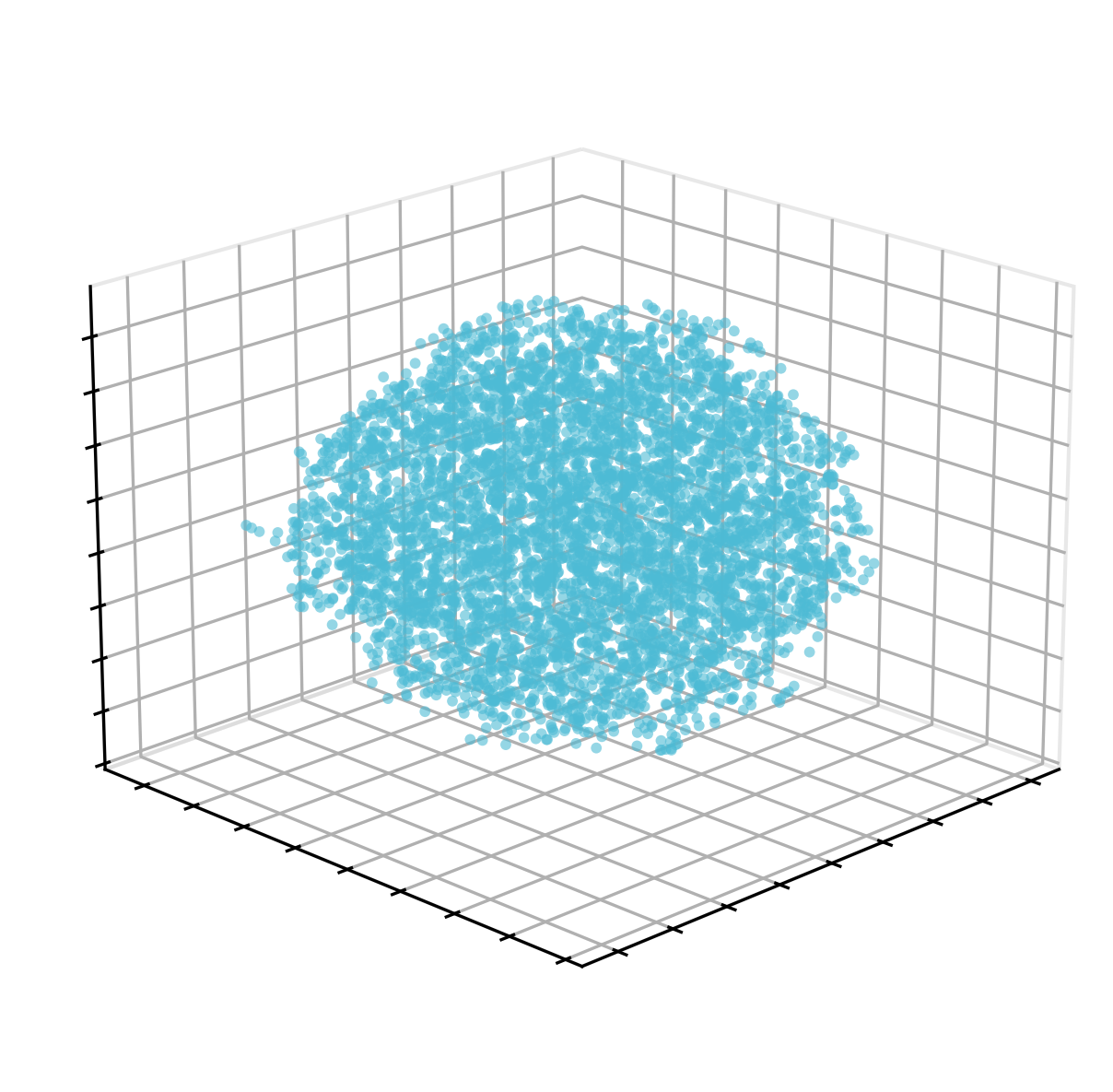}
        \vspace{-0.8cm}
        \caption{\tiny{Gaussian representations $\bm{z}_0$}}
        \label{fig:umap_z}
    \end{subfigure}
    \caption{\textbf{MeanFlow-based Representation Alignment.} We visualize the 3D UMAP projections of the contextualized embeddings $\bm{e}$ from the pretrained encoder, which exhibit complex structured manifolds, and the aligned representations $\bm{z}_0$ obtained via MeanFlow, which follow a standard Gaussian distribution while preserving semantic relations.}
    \label{fig:meanflow_umap}
\end{figure}

\begin{algorithm}[t]
    \caption{Two-stage training of RMDM}
    \label{alg:train}
    \begin{algorithmic}[1]
    \REQUIRE Dataset $\mathcal{D}$, pretrained encoder $g_\psi$, RMDM $p_\theta$, MeanFlow network $u_\phi$
    \STATE \textbf{Stage I: Train MeanFlow in representation space.}
    \FOR{each minibatch $\bm{x}_0 \sim \mathcal{D}$}
        \STATE $\bm{z}_1 \gets g_\psi(\bm{x}_0)$ \COMMENT{contextualized embeddings as target}
        \STATE Sample $\bm{z}_0 \sim \mathcal{N}(\bm{0}, \bm{I})$ and $(r,t)$ \COMMENT{noise as prior}
        \STATE $\bm{z}_t \gets (1-t)\bm{z}_0 + t\bm{z}_1$, \quad $v_t \gets \bm{z}_1 - \bm{z}_0$
        \STATE Compute $(u_\phi(\bm{z}_t,r,t), \frac{\mathrm{d}}{\mathrm{d} t}u_\phi(\bm{z}_t,r,t))$ via JVP
        \STATE Update $\phi$ by minimizing $\mathcal{L}(\phi)$ in Equation~\ref{eq:meanflow_loss}
    \ENDFOR
    \STATE Freeze $\phi$ (and freeze $g$).
    \STATE \textbf{Stage II: Train RMDM conditioned on representations.}
    \FOR{each minibatch $\bm{x}_0 \sim \mathcal{D}$}
        \STATE $\bm{e} \gets g_\psi(\bm{x}_0)$
        \STATE $\bm{z}_0 \gets \bm{e} - u_\phi(\bm{e}, 0, 1)$ \COMMENT{map embedding to noise space}
        \STATE Sample $t$ and construct a masked input $\bm{x}_t$ by the MDM corruption process
        \STATE Update $\theta$ by minimizing $\mathcal{L}(\theta)$ in Equation~\ref{eq:mdlm_cond_loss}
    \ENDFOR
    \end{algorithmic}
\end{algorithm}

In this subsection, we introduce the Representation-based Masked Diffusion Model (RMDM). The training process includes two stages. First, we obtain the representation $\bm{z}$ from data $\bm{x}$ through a pretrained encoder and transform its distribution to a standard Gaussian distribution by training a MeanFlow model. Then, we use this representation $\bm{z}$ as a condition to train the MDM.
To operationalize the latent variable $\bm{z}$ and bridge the conditional dependency gap discussed in Section~\ref{sec:method:gap}, we leverage the rich semantic representations provided by pretrained encoder models. Encoder-only architectures, ranging from the seminal BERT ~\citep{bert} to modern large-scale embedding models such as Qwen3 ~\citep{qwen3embedding}, map a discrete sequence $\bm{x}$ to a contextualized continuous embedding $\bm{z}_1 = g_{\psi}(x) \in \mathbb{R}^{L \times D_z}$, where $g_{\psi}$ is the encoder model. These contextualized embeddings capture high-level semantic structures and provide a smoother, more informative representation, making them ideal candidates for the global latent.

However, a direct utilization of this embedding introduces a fundamental discrepancy between training and inference: while $\bm{z}_1$ is readily computable from the ground-truth sequence $\bm{x}$ during training, it remains unobserved during inference. To resolve this, in stage I, we push the embedding distribution to a prior distribution which is easier to sample from while keeping the semantic structure of the embedding. In particular, we adopt the MeanFlow framework~\citep{shi2024meanflow} to learn a mapping between the data embedding $\bm{z}_1$ and a standard Gaussian prior $\bm{z}_0 \sim \mathcal{N}(\bm{0}, \bm{I})$. MeanFlow builds on the linear interpolation path $\bm{z}_t = (1-t)\bm{z}_0 + t\bm{z}_1$ like flow matching~\citep{lipman2022flow}, but replaces the instantaneous velocity field with an \textit{average velocity} field $u(\bm{z}_t, r, t)$ defined over an interval $[r, t]$:
\begin{align*}
    u(\bm{z}_t, r, t) \triangleq \frac{1}{t-r} \int_{r}^{t} v(\bm{z}_\tau, \tau) \mathrm{d} \tau.
\end{align*}
We parameterize $u_\phi(\bm{z}_t, r, t)$ with a neural network and train it using the MeanFlow identity, which relates average velocity and instantaneous velocity. Concretely, letting $v_t \triangleq \bm{z}_1 - \bm{z}_0$ denote the conditional velocity associated with the linear path, the MeanFlow objective regresses $u_\phi$ to the effective target induced by the identity:
\begin{align}
    \resizebox{.99\linewidth}{!}{$
    \mathcal{L}(\phi) \triangleq \mathbb{E} \left[ \left\lVert u_\phi(\bm{z}_t, r, t) - \text{sg}\left( v_t - (t-r) \frac{\mathrm{d}}{\mathrm{d} t} u_\phi(\bm{z}_t, r, t) \right) \right\rVert^2 \right]
    $}
\label{eq:meanflow_loss}
\end{align}
where $\text{sg}$ denotes the stop-gradient operator and the total derivative $\frac{\mathrm{d}}{\mathrm{d} t}$ is computed efficiently via Jacobian-vector products (as summarized in Section~\ref{sec:pre:continuous}).

Once trained, MeanFlow enables efficient {one-step transport}: given a data embedding $\bm{z}_1$, we recover its corresponding representation $\bm{z}_0$ in the Gaussian space by inverting the flow over the full interval $[0, 1]$:
\begin{align*}
    \bm{z}_0 = \bm{z}_1 - u_\phi(\bm{z}_1, 0, 1).
\end{align*}
Figure~\ref{fig:meanflow_umap} shows the UMAP~\citep{mcinnes2018umap} visualization of embeddings $\bm{z}_1 \triangleq \bm{e}$ and $\bm{z}_0$ obtained via MeanFlow, where $\bm{e}$ has its complex intrinsic geometric structure while the distribution of $\bm{z}_0$ is stable and easy to sample from.
We then use this Gaussian representation to condition the masked diffusion language model. Specifically, during Stage~II training we compute the data-dependent latent $\bm{z}_0 = g_{\psi}(\bm{x}_0) - u_\phi(g_{\psi}(\bm{x}_0), 0, 1)$ using the frozen pretrained encoder $g_{\psi}$ and the trained MeanFlow network $u_\phi$. We then train the RMDM $p_\theta(\bm{x}_0^i \mid \bm{x}_t, \bm{z}_0)$ to reconstruct masked tokens. Using the standard MDM training objective in Equation~\ref{eq:mdm_elbo}, we optimize:
{\small
\begin{equation}
    \mathcal{L}(\theta) \triangleq \mathbb{E}_{t,\bm{x}_0,\bm{x}_t,\bm{z}_0}\left[\frac{1}{t}\sum_{i\in \mathcal{I}_t} -\log p_\theta(\bm{x}_0^i \mid \bm{x}_t, \bm{z}_0)\right].
    \label{eq:mdlm_cond_loss}
\end{equation}
}
where $\mathcal{I}_t=\{i:\bm{x}_t^i=\texttt{[MASK]}\}$.
At inference time, since we have aligned the representation space with a Gaussian prior, we simply sample $\bm{z}$ directly from the prior $\mathcal{N}(\bm{0}, \bm{I})$, which yields a global representation used to guide parallel token generation. In addition, during the training in stage two and inference, the integration of $z_0$ will not lead to additional huge computation costs.


\begin{algorithm}[t]
\caption{Sampling with RMDM}
\label{alg:sample}
\begin{algorithmic}[1]
\REQUIRE RMDM $p_\theta$, number of denoising steps $K$
\STATE Sample $\bm{z} \sim \mathcal{N}(\bm{0}, \bm{I})$
\STATE Initialize $\bm{x}^{(0)}$ as all \texttt{[MASK]}
\FOR{$k = 0,1,\dots,K-1$}
    \STATE Select a subset of masked positions to update (parallel schedule)
    \STATE Sample tokens in parallel: $\bm{x}^{(k+1)} \sim p_\theta(\cdot \mid \bm{x}^{(k)}, \bm{z})$ on selected positions
\ENDFOR
\STATE \textbf{return} $\bm{x}^{(K)}$
\end{algorithmic}
\end{algorithm}

\subsection{Theoretical Analysis: Decomposing the Generation Gap}
\label{sec:method:theory}

In this section, we provide a rigorous decomposition of the error associated with parallel decoding in RMDM. 


\begin{proposition}[Tri-part Decomposition of the Joint KL]
\label{prop:decomposition}
Let $p_{\phi}(\bm{z})$ be the distribution of latents derived from the MeanFlow transport (approximating the posterior), and define the induced joint
\begin{equation*}
    \resizebox{.98\linewidth}{!}{$
    p_{\text{data}}(\bm{x}_M,\bm{x}_U,\bm{z})
    \triangleq
    p_{\phi}(\bm{z})\,p_{\text{data}}(\bm{x}_U \mid \bm{z})\,
    p_{\text{data}}(\bm{x}_M \mid \bm{x}_U,\bm{z}),
    $}
\end{equation*}
RMDM generative samples factorized as
\begin{equation*}
    \resizebox{.98\linewidth}{!}{$
    p_\theta(\bm{x}_M,\bm{x}_U,\bm{z})
    =
    p(\bm{z})\,\prod_{i \in U} p_\theta(x_i \mid \bm{z})\,
    \prod_{i \in M} p_\theta(x_i \mid \bm{x}_U, \bm{z}).
    $}
\end{equation*}
where $p(z)$ is the prior distribution.
Then the joint divergence between the data distribution and the model admits the following tri-part decomposition. Throughout the statement, a conditional KL such as
$D_{\mathrm{KL}}(p(y\mid x)\|q(y\mid x))$ denotes the averaged quantity
$\mathbb{E}_{p(x)}D_{\mathrm{KL}}(p(y\mid x)\|q(y\mid x))$; total-correlation terms are averaged over their conditioning variables in the same way:
\begin{align}
    &D_{\mathrm{KL}}\!\left(
    p_{\text{data}}(\bm{x}_M,\bm{x}_U,\bm{z})
    \,\middle\|\,
    p_\theta(\bm{x}_M,\bm{x}_U,\bm{z})
    \right) \nonumber \\
    =\;
    &D_{\mathrm{KL}}\!\left(
    p_{\phi}(\bm{z})
    \,\middle\|\,
    p(\bm{z})
    \right)
    \nonumber \\
    +&
    \mathcal{T}(\bm{x}_U \mid \bm{z})
    +
    \mathcal{T}(\bm{x}_M \mid \bm{x}_U,\bm{z})
    \nonumber \\
    +&
    \sum_{i\in U}
    D_{\mathrm{KL}}\!\left(
    p_{\text{data}}(x_i \mid \bm{z})
    \,\middle\|\,
    p_\theta(x_i \mid \bm{z})
    \right)
    \nonumber \\
    +&
    \sum_{i\in M}
    D_{\mathrm{KL}}\!\left(
    p_{\text{data}}(x_i \mid \bm{x}_U,\bm{z})
    \,\middle\|\,
    p_\theta(x_i \mid \bm{x}_U,\bm{z})
    \right).
    \label{eq:joint_kl_decomposition}
\end{align}
The four lines after the equality correspond to (I) MeanFlow--prior mismatch,
(II) remaining dependency, and the two parts of (III) RMDM prediction error.
\end{proposition}
The proof is provided in Appendix~\ref{app:proof_prop_decomposition}.
\paragraph{Derivation and Analysis.}
We analyze each term in the decomposition to demonstrate how our framework systematically addresses the challenges of parallel generation in masked diffusion models.

\textbf{(I) MeanFlow--Prior Mismatch.}
The first term $D_{\mathrm{KL}}(p_{\phi}(\bm{z}) \| p(\bm{z}))$ quantifies the discrepancy between the aggregated posterior $p_{\phi}(\bm{z})$ (induced by MeanFlow transport from encoder embeddings) and the prior $p(\bm{z}) = \mathcal{N}(\bm{0}, \bm{I})$ used during generation. This mismatch is critical because the decomposition in Proposition~\ref{prop:decomposition} assumes that the ground truth statistical modeling is based on the posterior $p_{\phi}(\bm{z})$, while inference samples from $p_\theta(\bm{z})$. MeanFlow implicitly minimizes this term by learning a transport map that aligns the encoder-derived latent distribution with the Gaussian prior at the population level. Unlike standard variational autoencoders that enforce per-sample posterior constraints, MeanFlow's population-level alignment ensures that latents sampled from the prior during inference closely match the distribution of latents seen during training, thereby preserving the validity of the decomposition.

\textbf{(II) Remaining Dependency.}
The second term $\mathcal{T}(\bm{x}_U \mid \bm{z}) + \mathcal{T}(\bm{x}_M \mid \bm{x}_U, \bm{z})$ captures the conditional dependencies among tokens that remain after conditioning on the latent $\bm{z}$. These terms measure the irreducible error of the factorization model: even with perfect token-wise modeling, dependencies not explained by $\bm{z}$ will cause the factorized approximation to deviate from the true joint distribution. By leveraging a powerful pretrained encoder like Qwen3~\citep{qwen3embedding} to extract rich global semantics $\bm{z}$, we minimize these dependency terms (Proposition~\ref{prop:dependency_reduction}), making the factorized assumption $p(\bm{x}_M \mid \bm{x}_U, \bm{z}) = \prod_{i \in M} p(x_i \mid \bm{x}_U, \bm{z})$ approximately valid. The quality of this term is thus determined by the representational capacity of the pretrained encoder $g_{\psi}$.

\textbf{(III) RMDM Prediction Error.}
The third term $\sum_{i \in U \cup M} D_{\mathrm{KL}}(p_{\text{data}}(x_i \mid \cdot) \| p_\theta(x_i \mid \cdot))$ measures the token-wise prediction error of $p_\theta$ given the latent $\bm{z}$. This term corresponds directly to the RMDM training objective Equation~\ref{eq:mdlm_cond_loss}, which optimizes the model to reconstruct individual tokens conditioned on the masked context $\bm{x}_U$ and the global latent $\bm{z}$. By training a powerful diffusion model $p_\theta$ to minimize the ELBO, we implicitly reduce the prediction error. The quality of this term depends on the capacity and training of the RMDM $p_\theta$.
\paragraph{Summary.}
Proposition~\ref{prop:decomposition} provides a unified theoretical justification for our two-stage framework:
\begin{itemize}[leftmargin=*, noitemsep, topsep=0pt]
    \item We minimize \textbf{(I)} by aligning the latent space via MeanFlow transport, ensuring inference-time latents match training-time distributions.
    \item We minimize \textbf{(II)} by extracting rich global semantics $\bm{z}$ from a powerful pretrained encoder, reducing residual dependencies among tokens.
    \item We minimize \textbf{(III)} by training a powerful masked diffusion model $p_\theta$ to accurately model token distributions conditioned on $\bm{z}$.
\end{itemize}

\subsection{Model Architecture}
\label{sec:method:architecture}

\begin{table}[t]
    \centering
    \setlength{\tabcolsep}{3.5pt}
    \begin{tabular}{clcc}
    \toprule
      & Model & Iterations & PPL ($\downarrow$)\\
    \midrule
      \multirow{2}{*}{\rot{\emph{AR}}}
          & Transformer-X Base  & - & 23.5 \\
          & $\text{OmniNet}_T$  & - & 21.5 \\
    \midrule
      \multirow{6}{*}{\rot{\emph{Diffusion}}}
          & D3PM (absorb)  & 1M & $\leq$76.90 \\
          & Diffusion-LM  & 1M & $\leq$118.62 \\
          & DiffusionBert  & 1M & $\leq$63.78 \\
          & SEDD  & 1M & $\leq$ 32.79 \\
          & MDLM  & 10M & $\leq$23.00 \\
          & MDLM$^\dagger$\  & 1M & $\leq$27.60 \\
        & \textbf{RMDM} (ours) & 1M & $\leq$\textbf{17.00} \\
    \bottomrule
    \end{tabular}
    \caption{Test perplexities (PPL $\downarrow$) on LM1B. $\dagger$ denotes our retrained models, other results are reported in ~\citet{he2022diffusionbert} and ~\citet{sahoo2024mdlm}.
    Best value is bolded. 
    }
    \label{tab:lm1b-ppl}
\end{table}

\begin{table*}[!t]
    \centering
    \setlength{\tabcolsep}{3.2pt}
    \renewcommand{\arraystretch}{1.08}
    \begin{tabular*}{0.92\textwidth}{@{\extracolsep{\fill}}llrrrrrrr@{}}
        \toprule
        Model & Metric & 1024 & 512 & 256 & 128 & 64 & 32 & 16 \\
        \midrule
        \multirow{2}{*}{MDLM}
            & GenPPL$\downarrow$ & 49.68 & 59.71 & 72.85 & 84.89 & 105.33 & 138.19 & 213.44 \\
            & Judge$\uparrow$ & 2.36 & 2.29 & 2.12 & 2.06 & 2.04 & 1.95 & 1.73 \\
        \midrule
        \multirow{2}{*}{SEDD}
            & GenPPL$\downarrow$ & 52.00 & 60.79 & 75.52 & 88.77 & 108.88 & 142.31 & 226.39 \\
            & Judge$\uparrow$ & 2.29 & 2.23 & 2.17 & 2.07 & 1.93 & 1.93 & 1.72 \\
        \midrule
        \multirow{2}{*}{\textbf{RMDM}}
            & GenPPL$\downarrow$ & \textbf{34.58} & \textbf{40.48} & \textbf{45.08} & \textbf{52.06} & \textbf{61.54} & \textbf{85.84} & \textbf{126.58} \\
            & Judge$\uparrow$ & \textbf{2.41} & \textbf{2.44} & \textbf{2.34} & \textbf{2.29} & \textbf{2.17} & \textbf{2.12} & \textbf{1.87} \\
        \bottomrule
    \end{tabular*}
    \caption{Generation quality on OpenWebText. GenPPL is evaluated by GPT-2; Judge is the average LLM-judge overall-quality score.}
    \label{tab:owt_genppl_comparison}
\end{table*}

\paragraph{Backbone and Conditioning Mechanism.}
We employ a Diffusion Transformer (DiT) architecture~\citep{vaswani2017attention,peebles2023scalable} for both the MeanFlow network $u_\phi$ and the RMDM $p_\theta$. The model operates on a sequence of hidden states $\bm{h} \in \mathbb{R}^{L \times D}$, derived from discrete token inputs $\bm{x} \in \{1, \dots, V\}^{L}$ via a learned embedding table. To incorporate the global semantic context, we inject the latent variable $\bm{z} \in \mathbb{R}^{L \times D_z}$ into every transformer block using Adaptive Layer Normalization (AdaLN) ~\citep{xu2019understanding, peebles2023scalable}. Specifically, the layer input $\bm{h}$ is modulated as:
\begin{equation*}
    \resizebox{.98\linewidth}{!}{$
    f_{\text{AdaLN}}(\bm{h}; \bm{z}) \triangleq
    f\!\left(\bm{\gamma}(\bm{z}) \odot \mathrm{LN}(\bm{h}) + \bm{\beta}(\bm{z})\right)
    \odot \alpha(\bm{z}) + \bm{h}.
    $}
    \label{eq:adaln}
\end{equation*}
where $\bm{\gamma}(\cdot),\bm{\beta}(\cdot),\alpha(\cdot)$ are parameters predicted from $\bm{z}$ by a shallow MLP, $f(\cdot)$ is a multiplication-like operation that can be implemented as an attention or MLP. Unlike standard diffusion models that primarily condition on the timestep $t$, we ensure the denoising process is consistently guided by the global semantic representation $\bm{z}$ across all layers.

\paragraph{Latent Stabilization.}
\label{sec:method:stabilization}
Conditioning on high-dimensional latents from large pretrained encoders presents a challenge: the model may overfit to the training latents, leading to poor generalization on samples drawn from the prior $\mathcal{N}(\bm{0}, \bm{I})$ during inference. We mitigate this via a two-step stabilization strategy.
First, we project $\bm{z}$ channel-wise to a lower-dimensional space $\mathbb{R}^{L \times d_z}$ (where $d_z < D_z$) via an orthogonal linear map and then upsample it to the original channel dimension. This bottleneck keeps the model from focusing on robust semantic features and reduces the complexity of the conditioning signal, while the simplicity of the prior is well preserved.
Second, during training, we inject isotropic Gaussian noise $\bm{\epsilon} \sim \mathcal{N}(\bm{0}, \sigma^2 \bm{I})$ into the projected latent, which acts as a smoothing regularizer and encourages the model to be robust to local perturbations in the latent space. Empirically, these strategies improve the generalizability.

\section{Experiments}

\begin{table}[t]
        \centering
        \begin{tabular*}{0.9\linewidth}{@{\extracolsep{\fill}}lccc}
        \toprule
         & AR & MDLM & RMDM \\
        \midrule
        WikiText & \textbf{25.75} & 46.78 & \underline{36.70} \\
        LM1B & \underline{51.25} & 77.91 & \textbf{48.57} \\
        LAMBADA & \textbf{51.28} & 116.30 & \underline{70.71} \\
        AG News & \textbf{52.09} & 96.20 & \underline{61.24} \\
        Pubmed & \textbf{49.01} & 71.65 & \underline{51.45} \\
        Arxiv & \textbf{41.73} & 58.82 & \underline{45.51} \\
        \bottomrule
        \end{tabular*}
        \caption{Zero-shot perplexities ($\downarrow$) of models trained on OpenWebText (OWT).
        All perplexities for diffusion models are upper bounds.
        Best is in bold and the second best is underlined.
        }
        \label{zeroshot-ppl}
    \end{table}

\subsection{Experimental Setup}

\paragraph{Tasks and Datasets.}
We study unconditional text generation following the standard evaluation protocol for diffusion language models ~\citep{lou2024discrete, sahoo2024mdlm}.
We evaluate on One Billion Word (LM1B)~\citep{chelba2013lm1b} and OpenWebText (OWT)~\citep{Gokaslan2019OpenWeb}. For LM1B, we report test perplexity derived from Equation~\ref{eq:mdlm_cond_loss}. For OWT, we report zero-shot perplexity on WikiText~\citep{merity2017pointer}, LM1B, LAMBADA~\citep{paperno-etal-2016-lambada}, AG News~\citep{zhang2015character}, and Scientific Papers~\citep{Cohan_2018}. For generation quality, we evaluate 256 generated samples with GPT-2 GenPPL~\citep{gpt2}. Since GenPPL is an evaluator-perplexity proxy and may not fully reflect human-perceived generation quality~\citep{shi2024simplified}, we additionally evaluate overall quality with two independent LLM judges. Implementation details, including tokenization, latent encoders, model sizes, and optimization settings, are provided in Appendix~\ref{app:experimental_details}; training and inference costs are reported in Appendices~\ref{app:training_cost} and~\ref{app:inference_efficiency}; and the LLM-judge protocol and confidence intervals are provided in Appendix~\ref{app:llm_judge}.
RMDM and our retrained baselines use the same tokenizer, DiT backbone, data preprocessing, and optimization schedule. RMDM additionally uses representations from a pretrained encoder during training; the baselines do not receive equivalent external representation knowledge.

\subsection{Main Results}

\paragraph{GenPPL improvements are the strongest when steps are scarce.}
On OpenWebText, Table~\ref{tab:owt_genppl_comparison} shows that RMDM consistently improves GenPPL over MDM-style baselines like MDLM~\citep{sahoo2024mdlm} and SEDD~\citep{lou2024discrete}, and the gap widens as the sampling budget decreases. With only a few parallel refinement steps, masked diffusion often loses global coherence due to limited shared context; in contrast, RMDM injects a sequence-level representation at each denoising step to coordinate simultaneous token updates.
RMDM also obtains the strongest primary-judge scores across all sampling budgets, and an independent second judge yields the same overall ranking trend. This agreement suggests that the GenPPL gains are accompanied by improvements in the qualities targeted by the evaluation rubric rather than being specific to one evaluator.

\begin{table}[t]
    \centering
    \setlength{\tabcolsep}{4.0pt}
    \renewcommand{\arraystretch}{1.05}
    \begin{tabular}{@{}rrrr@{}}
        \toprule
        Target GenPPL & MDLM & RMDM & Gain \\
        \midrule
        $\leq 50$  & 80 tok/s   & 287 tok/s  & $3.6\times$ \\
        $\leq 75$  & 318 tok/s  & 1142 tok/s & $3.6\times$ \\
        $\leq 130$ & 1269 tok/s & 4442 tok/s & $3.5\times$ \\
        \bottomrule
    \end{tabular}
    \caption{Matched-quality throughput on one A800 GPU (batch size 8, $L=1024$).}
    \label{tab:iso_quality_efficiency}
\end{table}

\paragraph{Inference efficiency.}
Table~\ref{tab:iso_quality_efficiency} reports measured throughput at representative matched-quality operating points. Although latent conditioning makes RMDM about 10\% slower than MDLM at the same number of steps, RMDM reaches the same GenPPL with substantially fewer steps, yielding approximately $3.6\times$ higher throughput. Peak sampling memory increases only from 18.64 GB to 18.93 GB. The pretrained encoder and MeanFlow network are not run at inference; generation draws the latent once from the prior. Full measurements are reported in Appendix~\ref{app:inference_efficiency}.

\paragraph{Validation and zero-shot perplexity.}
Beyond sample quality, RMDM does not degrade likelihood-based evaluation on LM1B (Table~\ref{tab:lm1b-ppl}). We speculate that its strong LM1B performance even against AR baselines is partly due to the large fraction of \texttt{[PAD]} tokens in LM1B sequences: the sequence-level representation can identify padding structure and provide a shared global signal during denoising. After scaling on OpenWebText, it also improves zero-shot perplexity over diffusion baselines across diverse test sets (Table~\ref{zeroshot-ppl}), suggesting that the learned representation distribution generalizes beyond training data.


\subsection{Ablation Study}
\label{sec:ablation}

\begin{table}[t]
    \centering
    \setlength{\tabcolsep}{3.6pt}
    \renewcommand{\arraystretch}{1.05}
    \begin{tabular*}{0.92\linewidth}{@{\extracolsep{\fill}}lrrrr@{}}
        \toprule
        & \multicolumn{2}{c}{GenPPL $\downarrow$} & \multicolumn{2}{c}{Judge $\uparrow$} \\
        \cmidrule(lr){2-3}\cmidrule(l){4-5}
        Step & RMDM & Indep. $z$ & RMDM & Indep. $z$ \\
        \midrule
        1024 & \textbf{34.58} & 50.45 & \textbf{2.41} & 2.25 \\
        512  & \textbf{40.48} & 59.15 & \textbf{2.44} & 2.23 \\
        256  & \textbf{45.08} & 70.69 & \textbf{2.34} & 2.20 \\
        128  & \textbf{52.06} & 81.41 & \textbf{2.29} & 2.09 \\
        64   & \textbf{61.54} & 98.83 & \textbf{2.17} & 2.02 \\
        32   & \textbf{85.84} & 126.67 & \textbf{2.12} & 1.93 \\
        16   & \textbf{126.58} & 176.86 & \textbf{1.87} & 1.80 \\
        \bottomrule
    \end{tabular*}
    \caption{Ablation study on OpenWebText. Indep. $z$ denotes a variant trained with Gaussian conditioning latents sampled independently of the target sequence $x$.}
    \label{tab:ablation_z}
\end{table}

\paragraph{Source of the continuous latent $z$.}
Table~\ref{tab:ablation_z} isolates whether the gain comes from the data-dependent representation or merely from adding a Gaussian conditioning variable. We replace the encoded latent with an independent sample $z\sim\mathcal{N}(0,I)$ during the training process, while keeping RMDM interface, conditioning mechanism, and sampling schedule unchanged. This control has a similar marginal distribution to the aligned latent prior, but removes the sequence-level information that RMDM is designed to provide. It is consistently worse than RMDM across all sampling budgets, with particularly clear degradation under few-step sampling. This pattern shows that the extra conditioning pathway alone is not sufficient: a random Gaussian vector may match the prior distribution, but it cannot explain residual dependencies among simultaneously masked tokens. The useful signal is instead the data-dependent representation carried by $z$, which acts as a global plan and gives masked positions shared information for coordinating parallel token updates.


\section{Conclusion}

In this paper, we present {Representation-based Masked Diffusion Model} (RMDM), which extends masked diffusion language models with a global continuous latent variable to explicitly model sequence-level semantics and restore inter-token dependencies. This advantage of RMDM makes it possible to correctly generate multiple masked tokens simultaneously and is especially powerful with a small number of sampling steps, where standard parallel sampling methods are most prone to losing global coherence. The MeanFlow-based representation learning mechanism aligns pretrained encoder embeddings with a Gaussian prior, allowing the latent to be sampled once at inference without running the encoder or MeanFlow network. By conditioning the masked diffusion model on these latent representations, we couple the updating of all masked tokens through shared global semantic information, reducing the Conditional Dependency Gap in parallel generation. Experiments on LM1B and OpenWebText demonstrate that RMDM consistently improves generation quality, especially in aggressive few-step regimes, and achieves approximately $3.6\times$ higher measured throughput at matched quality in our A800 evaluation. Our approach demonstrates the importance of global semantic coordination in non-autoregressive generation and provides a principled framework for addressing the conditional dependency gap.


\section*{Limitations}

The experiments mainly focus on unconditional generation at model scales and computational budgets comparable to GPT-2-level. While the results provide evidence for the effectiveness of the approach in this setting, its transferability to substantially larger-scale models and conditional generation tasks remains to be examined. We leave these extensions for future work.

\section*{Ethics Statement}

This work focuses on the statistical modeling problem of masked diffusion language models, specifically the conditional independence assumption used in parallel masked-token updates. It does not involve human subjects, new data collection, or additional model training beyond the experimental setup described in the paper. All experiments are conducted on publicly available benchmarks. The proposed method does not introduce new ethical risks beyond those inherent to the underlying language models and the benchmark data on which they are evaluated.

\section*{Acknowledgments}

The authors acknowledge support from The Hong Kong Polytechnic University (Research Grant P0046811).

\bibliography{rmdm_emnlp2026}

\clearpage
\onecolumn
\appendix
\section{Related Works}
\label{app:related}

\paragraph{Discrete Diffusion Language Models.}
Discrete diffusion models generate text by reversing a corruption process defined directly on the discrete token space. Early research generally falls into two paradigms: transition-based frameworks and score-based methods. D3PM~\citep{austin2021structured} pioneered the adaptation of continuous diffusion concepts to discrete spaces via transition matrices, which was later extended to continuous time using Markov chains~\citep{campbell2022continuous}. Alternatively, inspired by score matching~\citep{song2020score}, several works proposed discrete counterparts to the Stein score for modeling data distributions~\citep{meng2022concrete, lou2024discrete}.
A dominant subclass within this domain is {Masked Diffusion Models (MDMs)}, which treat the corruption process as token masking, starting from a masked sequence and refining tokens simultaneously using bidirectional context. Subsequent studies~\citep{ou2025radd, sahoo2024mdlm, shi2024simplified} demonstrated that simplified masking mechanisms can significantly enhance performance, effectively bridging the gap between diffusion and autoregressive models. The iterative unmasking process inherent in MDMs supports sophisticated reasoning capabilities, such as iterative refinement~\citep{ired} and reverse-order reasoning~\citep{llada}. The framework has also been integrated with chain-of-thought reasoning~\citep{dot-sedd}, demonstrating strong performance in tasks requiring parallel context and systematic refinement. Similar algorithms are proposed from the flow matching perspective~\citep{discrete_flow_matching}. In addition to masking noise, some works attempt to leverage uniform noise, though these tend to yield inferior performance~\citep{gidd, ko_discrete_flow}. Recently, MDMs have been scaled to large language models; for instance, LLaDA~\citep{llada} scales up to 8 billion parameters, showcasing reasoning capabilities previously unseen in non-autoregressive models. Furthermore, Dream~\citep{dream7b} introduces a training paradigm that initializes diffusion models with pretrained autoregressive weights, combining the strengths of both approaches.

\paragraph{Continuous Diffusion Language Models.}
Continuous diffusion models (CDMs) reformulate text generation by performing diffusion in a continuous space, typically on word embeddings or logits. {Embedding-based Diffusion Models (EDMs)}, such as Diffusion-LM~\citep{diffusionlm}, apply Gaussian diffusion to the continuous embeddings of discrete tokens. This formulation naturally supports controllable generation and sequence-to-sequence tasks~\citep{categorical, tess, diffuseq}. Although early EDMs faced performance gaps compared to autoregressive models, Plaid~\citep{plaid} established empirical scaling laws that significantly narrowed the efficiency gap with autoregressive models. This framework was further extended by DoT-Plaid~\citep{dot-sedd} to support chain-of-thought reasoning by leveraging iterative latent refinement. Apart from embedding-based methods, some approaches operate on the logit space~\citep{ssdlm, duo} or explore multimodal integration~\citep{diffuseeverything, multiflow}. DUO~\citep{duo} attempts to connect two types of diffusion models via marginal matching and applies distillation tricks for continuous diffusion to discrete text diffusion. While CDMs benefit from the well-established theory of continuous diffusion, mapping the continuous latents back to discrete text remains a non-trivial challenge that often requires specific regularization or rounding strategies.

\paragraph{Latent-Augmented Masked Diffusion.}
To enhance the expressivity of discrete diffusion, a growing body of work investigates integrating continuous latent variables or auxiliary processes into the discrete generation framework. Several works have explored utilizing latent variable models to improve text modeling~\citep{bowman2015generating, kaiser2018fast}. Recently, \citet{kong2025scalable} used a latent variable structure for next-token prediction in autoregressive models, optimized with variational Bayes. In the context of diffusion, \citet{hayakawa2024distillation} considered distilling pretrained MDMs with latent variables as the backward transition by optimizing consistency loss.
A complementary direction introduces continuous variables to capture joint structure along the diffusion path. VADD~\citep{xie2025variationalautoencodingdiscretediffusion} associates a continuous latent with each reverse transition and trains a VAE-style model to obtain non-factorized posteriors. Concurrent to our work, CCDD~\citep{zhou2025coevolutionarycontinuousdiscretediffusion} jointly evolves continuous representations and discrete tokens during diffusion. RMDM instead samples an aligned latent once and keeps it fixed while denoising the discrete sequence. Thus, both methods use continuous representations to augment discrete diffusion, but assign them different roles in the sampling process.

\section{Theoretical Justifications}

\subsection{Proof of Proposition~\ref{prop:dependency_reduction}}
\label{app:proof_prop1}

For a set of indices $M$, write $\bm{x}_M=\{x_i:i\in M\}$.  For any
conditioning variable $C$, define the conditional dependency gap
\begin{equation}
    \mathcal{T}(\bm{x}_M \mid C)
    =
    D_{\mathrm{KL}}\!\left(
    p(\bm{x}_M \mid C)
    \,\middle\|\,
    \prod_{i\in M} p(x_i \mid C)
    \right),
\end{equation}
where the conditional KL is averaged over $C$ when $C$ is random:
\begin{equation}
    D_{\mathrm{KL}}\!\left(p(Y\mid C)\,\middle\|\,q(Y\mid C)\right)
    \triangleq
    \mathbb{E}_{p(C)}
    \mathbb{E}_{p(Y\mid C)}
    \left[
    \log \frac{p(Y\mid C)}{q(Y\mid C)}
    \right].
    \label{eq:conditional_kl_expectation_definition}
\end{equation}
Equivalently,
\begin{equation}
    \mathcal{T}(\bm{x}_M \mid C)
    =
    \sum_{i\in M} H(x_i \mid C) - H(\bm{x}_M \mid C).
    \label{eq:tc_entropy_identity}
\end{equation}

\paragraph{Restatement of Proposition~\ref{prop:dependency_reduction}.}
The residual gap $\mathcal{T}(\bm{x}_M \mid \bm{x}_U,\bm{z})$ is zero if and only if
\begin{equation}
    p(\bm{x}_M \mid \bm{x}_U,\bm{z})
    =
    \prod_{i\in M} p(x_i \mid \bm{x}_U,\bm{z})
    \quad \text{a.s.}
    \label{eq:conditional_independence_factorization}
\end{equation}
Moreover,
\begin{equation}
    \Delta \mathcal{T}
    \triangleq
    \mathcal{T}(\bm{x}_M \mid \bm{x}_U)
    -
    \mathcal{T}(\bm{x}_M \mid \bm{x}_U,\bm{z})
    =
    \sum_{i\in M} I(x_i;\bm{z}\mid \bm{x}_U)
    -
    I(\bm{x}_M;\bm{z}\mid \bm{x}_U).
    \label{eq:dependency_gap_reduction}
\end{equation}

\begin{proof}
The first claim follows directly from the non-negativity of KL divergence:
$\mathcal{T}(\bm{x}_M \mid \bm{x}_U,\bm{z})=0$ holds exactly when the two
conditional distributions in its definition are equal almost surely, which is
the factorization in~\eqref{eq:conditional_independence_factorization}.

For the second claim, apply~\eqref{eq:tc_entropy_identity} twice:
\begin{align}
    \Delta \mathcal{T}
    &=
    \left[
        \sum_{i\in M} H(x_i \mid \bm{x}_U)
        -
        H(\bm{x}_M \mid \bm{x}_U)
    \right]
    -
    \left[
        \sum_{i\in M} H(x_i \mid \bm{x}_U,\bm{z})
        -
        H(\bm{x}_M \mid \bm{x}_U,\bm{z})
    \right] \nonumber \\
    &=
    \sum_{i\in M}
    \left[
        H(x_i \mid \bm{x}_U)
        -
        H(x_i \mid \bm{x}_U,\bm{z})
    \right]
    -
    \left[
        H(\bm{x}_M \mid \bm{x}_U)
        -
        H(\bm{x}_M \mid \bm{x}_U,\bm{z})
    \right] \nonumber \\
    &=
    \sum_{i\in M} I(x_i;\bm{z}\mid \bm{x}_U)
    -
    I(\bm{x}_M;\bm{z}\mid \bm{x}_U).
\end{align}
\end{proof}

\subsection{Proof of Proposition~\ref{prop:decomposition}}
\label{app:proof_prop_decomposition}

We use the following factorization for the joint distribution induced by the
encoder-side latent variable:
\begin{align}
    p_{\text{data}}(\bm{x}_M,\bm{x}_U,\bm{z})
    &=
    p_{\phi}(\bm{z})\,
    p_{\text{data}}(\bm{x}_U \mid \bm{z})\,
    p_{\text{data}}(\bm{x}_M \mid \bm{x}_U,\bm{z}),
    \label{eq:joint_factor_data} \\
    p_{\theta}(\bm{x}_M,\bm{x}_U,\bm{z})
    &=
    p(\bm{z})\,
    \prod_{i\in U} p_{\theta}(x_i \mid \bm{z})\,
    \prod_{i\in M} p_{\theta}(x_i \mid \bm{x}_U,\bm{z}).
    \label{eq:joint_factor_model}
\end{align}

\paragraph{Restatement of Proposition~\ref{prop:decomposition}.}
Under~\eqref{eq:joint_factor_data}--\eqref{eq:joint_factor_model}, the joint KL
decomposes as
\begin{align}
    &D_{\mathrm{KL}}\!\left(
    p_{\text{data}}(\bm{x}_M,\bm{x}_U,\bm{z})
    \,\middle\|\,
    p_{\theta}(\bm{x}_M,\bm{x}_U,\bm{z})
    \right)
    \label{eq:joint_kl_decomposition_appendix} \\
    =&
    \underbrace{
    D_{\mathrm{KL}}\!\left(p_{\phi}(\bm{z})\,\middle\|\,p(\bm{z})\right)
    }_{\text{(I) MeanFlow--prior mismatch}}
    +
    \underbrace{
    \mathcal{T}(\bm{x}_U\mid \bm{z})
    +
    \mathcal{T}(\bm{x}_M\mid \bm{x}_U,\bm{z})
    }_{\text{(II) residual conditional dependence}} \nonumber \\
    &\quad+
    \underbrace{
    \sum_{i\in U}
    D_{\mathrm{KL}}\!\left(
    p_{\text{data}}(x_i\mid \bm{z})
    \,\middle\|\,
    p_{\theta}(x_i\mid \bm{z})
    \right)
    +
    \sum_{i\in M}
    D_{\mathrm{KL}}\!\left(
    p_{\text{data}}(x_i\mid \bm{x}_U,\bm{z})
    \,\middle\|\,
    p_{\theta}(x_i\mid \bm{x}_U,\bm{z})
    \right)
    }_{\text{(III) token-wise prediction error}},
    \nonumber
\end{align}
where the conditional KL terms are averaged over their conditioning variables.

\begin{proof}
First, substituting the factorizations in
\eqref{eq:joint_factor_data}--\eqref{eq:joint_factor_model} into the joint KL
gives
\begin{align}
    &D_{\mathrm{KL}}\!\left(
    p_{\text{data}}(\bm{x}_M,\bm{x}_U,\bm{z})
    \,\middle\|\,
    p_{\theta}(\bm{x}_M,\bm{x}_U,\bm{z})
    \right) \nonumber \\
    &=
    \mathbb{E}_{p_{\text{data}}(\bm{x}_M,\bm{x}_U,\bm{z})}
    \left[
    \log
    \frac{
    p_{\phi}(\bm{z})
    p_{\text{data}}(\bm{x}_U\mid \bm{z})
    p_{\text{data}}(\bm{x}_M\mid \bm{x}_U,\bm{z})
    }{
    p(\bm{z})
    \prod_{i\in U}p_{\theta}(x_i\mid \bm{z})
    \prod_{i\in M}p_{\theta}(x_i\mid \bm{x}_U,\bm{z})
    }
    \right].
    \label{eq:joint_kl_log_ratio}
\end{align}
Separating the three log-ratio terms yields
\begin{align}
    &D_{\mathrm{KL}}\!\left(
    p_{\text{data}}(\bm{x}_M,\bm{x}_U,\bm{z})
    \,\middle\|\,
    p_{\theta}(\bm{x}_M,\bm{x}_U,\bm{z})
    \right) \nonumber \\
    &=
    D_{\mathrm{KL}}\!\left(p_{\phi}(\bm{z})\,\middle\|\,p(\bm{z})\right)
    +
    D_{\mathrm{KL}}\!\left(
    p_{\text{data}}(\bm{x}_U\mid \bm{z})
    \,\middle\|\,
    \prod_{i\in U}p_{\theta}(x_i\mid \bm{z})
    \right) \nonumber \\
    &\quad+
    D_{\mathrm{KL}}\!\left(
    p_{\text{data}}(\bm{x}_M\mid \bm{x}_U,\bm{z})
    \,\middle\|\,
    \prod_{i\in M}p_{\theta}(x_i\mid \bm{x}_U,\bm{z})
    \right).
    \label{eq:kl_chain_rule_three_terms}
\end{align}

We next use the following identity to separate each conditional KL into a
dependency term and token-wise prediction terms. For any index set $A$ and
conditioning variable $C$, all conditional KLs below follow the averaged
definition in~\eqref{eq:conditional_kl_expectation_definition}:
\begin{align}
    &D_{\mathrm{KL}}\!\left(
    p_{\text{data}}(\bm{x}_A\mid C)
    \,\middle\|\,
    \prod_{i\in A}p_{\theta}(x_i\mid C)
    \right) \nonumber \\
    &=
    D_{\mathrm{KL}}\!\left(
    p_{\text{data}}(\bm{x}_A\mid C)
    \,\middle\|\,
    \prod_{i\in A}p_{\text{data}}(x_i\mid C)
    \right)
    +
    \sum_{i\in A}
    D_{\mathrm{KL}}\!\left(
    p_{\text{data}}(x_i\mid C)
    \,\middle\|\,
    p_{\theta}(x_i\mid C)
    \right).
    \label{eq:conditional_kl_decomposition}
\end{align}
To see this, start from the definition of KL and add and subtract
$\sum_{i\in A}\log p_{\text{data}}(x_i\mid C)$ inside the expectation:
\begin{align}
    &\mathbb{E}_{p(C)}
    \mathbb{E}_{p_{\text{data}}(\bm{x}_A\mid C)}
    \left[
    \log
    \frac{
    p_{\text{data}}(\bm{x}_A\mid C)
    }{
    \prod_{i\in A}p_{\theta}(x_i\mid C)
    }
    \right] \nonumber \\
    &=
    \mathbb{E}_{p(C)}
    \mathbb{E}_{p_{\text{data}}(\bm{x}_A\mid C)}
    \left[
    \log
    \frac{
    p_{\text{data}}(\bm{x}_A\mid C)
    }{
    \prod_{i\in A}p_{\text{data}}(x_i\mid C)
    }
    \right] \nonumber \\
    &\quad+
    \mathbb{E}_{p(C)}
    \mathbb{E}_{p_{\text{data}}(\bm{x}_A\mid C)}
    \left[
    \log
    \frac{
    \prod_{i\in A}p_{\text{data}}(x_i\mid C)
    }{
    \prod_{i\in A}p_{\theta}(x_i\mid C)
    }
    \right] \nonumber \\
    &=
    \mathcal{T}(\bm{x}_A\mid C)
    +
    \sum_{i\in A}
    D_{\mathrm{KL}}\!\left(
    p_{\text{data}}(x_i\mid C)
    \,\middle\|\,
    p_{\theta}(x_i\mid C)
    \right).
    \label{eq:conditional_kl_decomposition_expanded}
\end{align}
The last equality uses the definition of conditional total correlation for the
first term and marginalizes $p_{\text{data}}(\bm{x}_A\mid C)$ to
$p_{\text{data}}(x_i\mid C)$ in each token-wise term.

Applying~\eqref{eq:conditional_kl_decomposition} with
$(A,C)=(U,\bm{z})$ gives
\begin{align}
    &D_{\mathrm{KL}}\!\left(
    p_{\text{data}}(\bm{x}_U\mid \bm{z})
    \,\middle\|\,
    \prod_{i\in U}p_{\theta}(x_i\mid \bm{z})
    \right) \nonumber \\
    &=
    \mathcal{T}(\bm{x}_U\mid \bm{z})
    +
    \sum_{i\in U}
    D_{\mathrm{KL}}\!\left(
    p_{\text{data}}(x_i\mid \bm{z})
    \,\middle\|\,
    p_{\theta}(x_i\mid \bm{z})
    \right),
    \label{eq:visible_kl_decomposition}
\end{align}
where the right-hand side is averaged over $\bm{z}\sim p_{\phi}(\bm{z})$.
Applying the same identity with $(A,C)=(M,(\bm{x}_U,\bm{z}))$ gives
\begin{align}
    &D_{\mathrm{KL}}\!\left(
    p_{\text{data}}(\bm{x}_M\mid \bm{x}_U,\bm{z})
    \,\middle\|\,
    \prod_{i\in M}p_{\theta}(x_i\mid \bm{x}_U,\bm{z})
    \right) \nonumber \\
    &=
    \mathcal{T}(\bm{x}_M\mid \bm{x}_U,\bm{z})
    +
    \sum_{i\in M}
    D_{\mathrm{KL}}\!\left(
    p_{\text{data}}(x_i\mid \bm{x}_U,\bm{z})
    \,\middle\|\,
    p_{\theta}(x_i\mid \bm{x}_U,\bm{z})
    \right),
    \label{eq:masked_kl_decomposition}
\end{align}
where the right-hand side is averaged over
$(\bm{x}_U,\bm{z})\sim p_{\text{data}}(\bm{x}_U,\bm{z})$.

Finally, substituting \eqref{eq:visible_kl_decomposition} and
\eqref{eq:masked_kl_decomposition} into
\eqref{eq:kl_chain_rule_three_terms} yields
\begin{align}
    &D_{\mathrm{KL}}\!\left(
    p_{\text{data}}(\bm{x}_M,\bm{x}_U,\bm{z})
    \,\middle\|\,
    p_{\theta}(\bm{x}_M,\bm{x}_U,\bm{z})
    \right) \nonumber \\
    =&
    \underbrace{
    D_{\mathrm{KL}}\!\left(p_{\phi}(\bm{z})\,\middle\|\,p(\bm{z})\right)
    }_{\text{(I) MeanFlow--prior mismatch}}
    +
    \underbrace{
    \mathcal{T}(\bm{x}_U\mid \bm{z})
    +
    \mathcal{T}(\bm{x}_M\mid \bm{x}_U,\bm{z})
    }_{\text{(II) residual conditional dependence}} \nonumber \\
    &\quad+
    \underbrace{
    \sum_{i\in U}
    D_{\mathrm{KL}}\!\left(
    p_{\text{data}}(x_i\mid \bm{z})
    \,\middle\|\,
    p_{\theta}(x_i\mid \bm{z})
    \right)
    +
    \sum_{i\in M}
    D_{\mathrm{KL}}\!\left(
    p_{\text{data}}(x_i\mid \bm{x}_U,\bm{z})
    \,\middle\|\,
    p_{\theta}(x_i\mid \bm{x}_U,\bm{z})
    \right)
    }_{\text{(III) token-wise prediction error}},
    \nonumber
\end{align}
which is Equation~\eqref{eq:joint_kl_decomposition_appendix}.
\end{proof}

\section{Experimental Details}
\label{app:experimental_details}

\paragraph{Data preprocessing.}
For LM1B, following~\citet{he2022diffusionbert, lou2024discrete}, we use the standard train/test split and tokenize the corpus with the BERT tokenizer~\citep{bert}. We pad and truncate sequences to a fixed length of $N=128$. For OWT, following~\citet{lou2024discrete, sahoo2024mdlm}, we reserve the last 100K documents as a held-out evaluation set. To match the encoder used for OWT, we tokenize with the Qwen2 tokenizer and form sequences of length $N=1024$.

\paragraph{Embedding models for the continuous latent.}
We instantiate the pretrained embedding model $g_\phi$ differently across datasets to match their tokenization and domain. On LM1B, we use BERT-Base~\citep{bert}. Because the BERT embedding is high-dimensional, we split each 128-token sequence into four groups, average embeddings within each group, and apply PCA whitening to obtain a compact 128-dimensional representation from the original 768-dimensional embeddings. On OWT, we use Qwen3-Embedding-0.6B~\citep{qwen3embedding}; because it supports flexible output dimensionality, we use the last-layer embedding with dimension 32.

\paragraph{Models and optimization.}
As described in Section~\ref{sec:method:architecture}, both the discrete denoising backbone and the MeanFlow backbone are parameterized by a Diffusion Transformer (DiT)~\citep{peebles2023scalable}, following the architecture used in~\citet{lou2024discrete}. We use 12 layers, hidden dimension 768, 12 attention heads, and $\sigma=1$. The latent dimensions are $D_z=128, d_z=32$ for LM1B and $D_z=32, d_z=16$ for OWT. We follow the common diffusion-LM training recipe~\citep{sahoo2024mdlm}: AdamW with a constant learning rate of $3\times10^{-4}$ after 2.5K warm-up iterations, together with exponential moving average (EMA) decay 0.9999. For LM1B, we train for 1M iterations with batch size 512. For OWT, we train for 50K iterations with batch size 512. The MDLM and SEDD results reported on OWT are from our own retrained baselines, using the same Qwen2 tokenizer, data preprocessing, model size, batch size, optimizer, learning-rate schedule, EMA setting, and 50K-step training budget as RMDM.

\paragraph{Encoder sensitivity.}
We replace Qwen3-Embedding-0.6B with ModernBERT-large~\citep{modernbert} while keeping the downstream architecture and latent interface fixed. ModernBERT supports the full OWT sequence length, and we reduce its 1024-dimensional outputs to 32 dimensions with PCA. Because this variant converges more slowly, we compare it with an MDLM baseline trained for the same approximately 60K-step budget; the Qwen3 results use the 50K-step models from the main experiment. Tables~\ref{tab:encoder_ablation} and~\ref{tab:encoder_ablation_judge} show that both encoder variants outperform their matched baselines across sampling budgets in terms of GenPPL and LLM-judge scores, although the improvement is smaller with ModernBERT. This result suggests that the benefit is not specific to one encoder family, while its magnitude depends on the representation source.

\begin{table}[h]
    \centering
    \setlength{\tabcolsep}{4.2pt}
    \renewcommand{\arraystretch}{1.05}
    \begin{tabular}{@{}lrrrrrrr@{}}
        \toprule
        Model & 1024 & 512 & 256 & 128 & 64 & 32 & 16 \\
        \midrule
        MDLM (50K) & 49.68 & 59.71 & 72.85 & 84.89 & 105.33 & 138.19 & 213.44 \\
        RMDM--Qwen3 (50K) & \textbf{34.58} & \textbf{40.48} & \textbf{45.08} & \textbf{52.06} & \textbf{61.54} & \textbf{85.84} & \textbf{126.58} \\
        MDLM (60K) & 48.17 & 55.61 & 67.06 & 81.20 & 103.66 & 134.82 & 218.39 \\
        RMDM--ModernBERT (60K) & \textbf{46.07} & \textbf{51.09} & \textbf{61.33} & \textbf{74.59} & \textbf{89.71} & \textbf{124.86} & \textbf{191.71} \\
        \bottomrule
    \end{tabular}
    \caption{Encoder sensitivity on OWT measured by GenPPL ($\downarrow$). Each RMDM variant is compared with the MDLM baseline at the corresponding training budget.}
    \label{tab:encoder_ablation}
\end{table}

\begin{table}[h]
    \centering
    \scriptsize
    \begin{tabular}{@{}lrrrrrrr@{}}
        \toprule
        Model & 1024 & 512 & 256 & 128 & 64 & 32 & 16 \\
        \midrule
        MDLM (50K) & $2.359{\pm}.078$ & $2.285{\pm}.072$ & $2.121{\pm}.064$ & $2.059{\pm}.070$ & $2.035{\pm}.063$ & $1.953{\pm}.056$ & $1.730{\pm}.059$ \\
        RMDM--Qwen3 (50K) & $\mathbf{2.406{\pm}.084}$ & $\mathbf{2.441{\pm}.085}$ & $\mathbf{2.340{\pm}.075}$ & $\mathbf{2.289{\pm}.069}$ & $\mathbf{2.168{\pm}.063}$ & $\mathbf{2.117{\pm}.055}$ & $\mathbf{1.867{\pm}.057}$ \\
        MDLM (60K) & $2.348{\pm}.079$ & $2.246{\pm}.077$ & $2.227{\pm}.070$ & $2.125{\pm}.070$ & $1.992{\pm}.063$ & $1.879{\pm}.055$ & $1.762{\pm}.056$ \\
        RMDM--ModernBERT (60K) & $\mathbf{2.391{\pm}.083}$ & $\mathbf{2.402{\pm}.085}$ & $\mathbf{2.277{\pm}.075}$ & $\mathbf{2.211{\pm}.072}$ & $\mathbf{2.102{\pm}.064}$ & $\mathbf{1.945{\pm}.055}$ & $\mathbf{1.832{\pm}.057}$ \\
        \bottomrule
    \end{tabular}
    \caption{Encoder sensitivity on OWT measured by LLM-judge overall-quality score ($\uparrow$; mean $\pm$ 95\% confidence-interval half-width). Each RMDM variant is compared with the MDLM baseline at the corresponding training budget.}
    \label{tab:encoder_ablation_judge}
\end{table}

\section{LLM-judge Evaluation}
\label{app:llm_judge}

We use an LLM judge to complement GenPPL because evaluator perplexity is only a proxy for generation quality and can miss aspects such as discourse coherence, degenerate repetition, and natural readability. Following this motivation, each generated passage is judged as a standalone sample. We use DeepSeek-V4-Flash~\citep{deepseekai2026deepseekv4} as the primary judge model and report the average \texttt{overall} score in the main tables. We additionally evaluate the same samples with Gemini-3.1-Flash-Lite~\citep{gemini31flashlite} as an independent judge. For each method and sampling budget, both judges evaluate all 256 generated passages. Tables~\ref{tab:judge_ci_deepseek} and~\ref{tab:judge_ci_gemini} report mean scores with 95\% confidence intervals. The agreement in relative rankings indicates that the main trend is not specific to a single judge. The judge is asked to return minified JSON with five 1--10 scores: \texttt{fluency}, \texttt{coherence}, \texttt{repetition}, \texttt{readability}, and \texttt{overall}.

\begin{table}[h]
    \centering
    \scriptsize
    \begin{tabular}{@{}lrrrrrrr@{}}
        \toprule
        Model & 1024 & 512 & 256 & 128 & 64 & 32 & 16 \\
        \midrule
        MDLM & $2.359{\pm}.078$ & $2.285{\pm}.072$ & $2.121{\pm}.064$ & $2.059{\pm}.070$ & $2.035{\pm}.063$ & $1.953{\pm}.056$ & $1.730{\pm}.059$ \\
        SEDD & $2.293{\pm}.083$ & $2.227{\pm}.075$ & $2.172{\pm}.064$ & $2.074{\pm}.066$ & $1.930{\pm}.058$ & $1.926{\pm}.063$ & $1.719{\pm}.059$ \\
        RMDM & $2.406{\pm}.084$ & $2.441{\pm}.085$ & $2.340{\pm}.075$ & $2.289{\pm}.069$ & $2.168{\pm}.063$ & $2.117{\pm}.055$ & $1.867{\pm}.057$ \\
        \bottomrule
    \end{tabular}
    \caption{DeepSeek-V4-Flash overall-quality scores (mean $\pm$ 95\% confidence-interval half-width; $n=256$ per cell).}
    \label{tab:judge_ci_deepseek}
\end{table}

\begin{table}[h]
    \centering
    \scriptsize
    \begin{tabular}{@{}lrrrrrrr@{}}
        \toprule
        Model & 1024 & 512 & 256 & 128 & 64 & 32 & 16 \\
        \midrule
        MDLM & $1.594{\pm}.061$ & $1.434{\pm}.061$ & $1.297{\pm}.056$ & $1.240{\pm}.053$ & $1.113{\pm}.039$ & $1.020{\pm}.017$ & $1.000{\pm}.000$ \\
        SEDD & $1.625{\pm}.060$ & $1.473{\pm}.062$ & $1.371{\pm}.060$ & $1.277{\pm}.055$ & $1.156{\pm}.045$ & $1.055{\pm}.028$ & $1.008{\pm}.011$ \\
        RMDM & $1.781{\pm}.051$ & $1.801{\pm}.049$ & $1.637{\pm}.059$ & $1.590{\pm}.061$ & $1.408{\pm}.061$ & $1.184{\pm}.048$ & $1.027{\pm}.020$ \\
        \bottomrule
    \end{tabular}
    \caption{Gemini-3.1-Flash-Lite overall-quality scores (mean $\pm$ 95\% confidence-interval half-width; $n=256$ per cell).}
    \label{tab:judge_ci_gemini}
\end{table}

\paragraph{Prompt.}
\begin{quote}
\small\ttfamily
Please act as an impartial judge and evaluate the quality of the generated text displayed below...

Scoring rubric, each from 1 to 10:

- fluency: grammar, phrasing, and local readability

- coherence: logical flow, topic consistency, and sentence-to-sentence continuity

- repetition: avoidance of loops, duplicate ideas, and degenerate repetition

- readability: how natural and human-readable the passage feels overall

- overall: your final holistic judgment

Instructions:

- Judge the text as a standalone passage.

- Do not assume hidden context beyond the text itself.

- Penalize malformed special tokens, obvious hallucinated structure, or text that reads like corrupted sampling.

- Do not reward verbosity.

Return ONLY valid minified JSON...

Generated text:

<<<TEXT

\{text\}

TEXT>>>
\end{quote}

\section{Inference Efficiency}
\label{app:inference_efficiency}

We measure sampling efficiency on one NVIDIA A800 80GB GPU with the DDPM sampler, batch size 8, and sequence length 1024, after warm-up. RMDM and MDLM share the same DiT backbone; RMDM additionally conditions each block on a latent sampled once from the prior. The pretrained encoder and MeanFlow network are not invoked during inference. RMDM uses 2.72 GB of weight memory and 18.93 GB peak sampling memory, compared with 2.43 GB and 18.64 GB for MDLM.

\begin{table}[h]
    \centering
    \setlength{\tabcolsep}{5pt}
    \begin{tabular}{@{}rrrrr@{}}
        \toprule
        Steps & MDLM GenPPL & MDLM tok/s & RMDM GenPPL & RMDM tok/s \\
        \midrule
        1024 & 49.68 & 80   & 34.58  & 72 \\
        512  & 59.71 & 159  & 40.48  & 144 \\
        256  & 72.85 & 318  & 45.08  & 287 \\
        128  & 84.89 & 636  & 52.06  & 573 \\
        64   & 105.33 & 1269 & 61.54  & 1142 \\
        32   & 138.19 & 2520 & 85.84  & 2263 \\
        16   & 213.44 & 4951 & 126.58 & 4442 \\
        \bottomrule
    \end{tabular}
    \caption{Generation quality and measured throughput at matched sampling-step budgets.}
    \label{tab:full_inference_efficiency}
\end{table}

At the same step count, RMDM retains approximately 90\% of MDLM's throughput. At matched quality, however, RMDM requires roughly one quarter as many steps: for example, GenPPL $\leq 50$ is reached by MDLM at 1024 steps (80 tok/s) and by RMDM at 256 steps (287 tok/s), corresponding to a $3.6\times$ throughput improvement. Similar comparisons give $3.5$--$3.6\times$ improvements across the quality range reported in Table~\ref{tab:iso_quality_efficiency}.

\section{Training Cost}
\label{app:training_cost}

\paragraph{Wall-clock training time.}
We trained the models on 4 NVIDIA A800 GPUs. In our current implementation,
the MDLM baseline requires approximately 3 days and 5 hours for 50K training
steps, while the RMDM decoder training stage requires approximately 5 days for
the same number of steps. This measured overhead mainly comes from our
uncached implementation: at each training step, we compute the pretrained
encoder representation online and then map it through the MeanFlow-aligned
latent module before conditioning the decoder. Since the training corpus is
fixed, these encoder representations can in principle be precomputed and
cached, which would remove most of this repeated computation. We did not use
such caching in our reported runs due to storage constraints.

\paragraph{MeanFlow alignment stage.}
The additional latent-alignment stage is comparatively lightweight. It is
trained for 10K steps and takes approximately 1 day and 15 hours under the same
hardware setting. This number should also be interpreted as the cost of the
uncached implementation, since online encoder representation extraction
contributes to the measured time. Overall, the wall-clock times above report
our implementation cost rather than an inherent training-time lower bound of
RMDM.

\section{Examples}
\label{app:examples}
\subsection{Detailed Numerical Example with Gaussian Distribution}
\label{app:example_calculation}

We provide a concrete analytical example using Gaussian distributions that strictly follows the causal graph in Figure~\ref{fig:dependency_graph}. This example demonstrates how the latent variable $\bm{z}$ explains correlations between masked tokens that the partial context $\bm{x}_U$ cannot capture.

\paragraph{Model Setup.}
Consider a linear Gaussian structural equation model where the global latent $\bm{z}$, the observed context $\bm{x}_U$, and the masked tokens $x_M^1, x_M^2$ are scalar random variables:
\begin{align}
    \bm{z} &\sim \mathcal{N}(0, 1), \\
    \bm{x}_U &= \bm{z} + \epsilon_U, \\
    x_M^1 &= \bm{z} + \alpha \bm{x}_U + \epsilon_1, \\
    x_M^2 &= \bm{z} + \beta \bm{x}_U + \epsilon_2,
\end{align}
where $\epsilon_U, \epsilon_1, \epsilon_2 \sim \mathcal{N}(0, \sigma^2)$ are independent noise terms. This structure matches the causal graph where $\bm{z}$ acts as a confounder influencing all variables, while $\bm{x}_U$ has direct edges (controlled by $\alpha$ and $\beta$) to the masked tokens.

\paragraph{Conditional Dependency Without Latent.}
When $\bm{z}$ is unobserved (standard non-latent modeling), we condition only on $\bm{x}_U$. Although $\bm{x}_U$ provides partial information about $\bm{z}$, it does not fully recover it. The posterior distribution is:
\begin{equation}
    p(\bm{z} \mid \bm{x}_U) = \mathcal{N}(\mu_{z|u}, \sigma^2_{z|u}), \quad \text{where } \sigma^2_{z|u} = \text{Var}(\bm{z} \mid \bm{x}_U) = \frac{\sigma^2}{1 + \sigma^2}.
\end{equation}
Since $\bm{z}$ remains uncertain (variance $\sigma^2_{z|u} > 0$), it acts as a common noise source inducing correlation between $x_M^1$ and $x_M^2$. Given $\bm{x}_U$, the terms $\alpha \bm{x}_U$ and $\beta \bm{x}_U$ are constant, so the conditional covariance is:
\begin{equation}
    \text{Cov}(x_M^1, x_M^2 \mid \bm{x}_U) = \text{Var}(\bm{z} \mid \bm{x}_U) = \frac{\sigma^2}{1 + \sigma^2}.
\end{equation}

Setting $\sigma = 1$ for concreteness, we compute:
\begin{align}
    \text{Var}(\bm{z} \mid \bm{x}_U) &= \frac{1}{2}, \\
    \text{Cov}(x_M^1, x_M^2 \mid \bm{x}_U) &= \frac{1}{2}, \\
    \text{Var}(x_M^i \mid \bm{x}_U) &= \text{Var}(\bm{z} \mid \bm{x}_U) + \text{Var}(\epsilon_i) = \frac{1}{2} + 1 = \frac{3}{2}, \\
    \rho &= \frac{\text{Cov}(x_M^1, x_M^2 \mid \bm{x}_U)}{\sqrt{\text{Var}(x_M^1 \mid \bm{x}_U)\text{Var}(x_M^2 \mid \bm{x}_U)}} = \frac{1/2}{3/2} = \frac{1}{3}.
\end{align}
For Gaussian variables, the dependency gap (mutual information) is:
\begin{equation}
    \mathcal{T}(x_M^1, x_M^2 \mid \bm{x}_U) = -\frac{1}{2} \ln(1 - \rho^2) = -\frac{1}{2} \ln\left(1 - \frac{1}{9}\right) > 0.
\end{equation}
This positive value quantifies the information loss incurred by assuming independence in parallel sampling.

\paragraph{Conditional Independence With Latent.}
In our RMDM framework, we condition on both $\bm{x}_U$ and the sampled latent $\bm{z}$. When $\bm{z}$ is fixed, the only remaining randomness comes from the independent noise terms $\epsilon_1, \epsilon_2$. Therefore:
\begin{equation}
    \text{Cov}(x_M^1, x_M^2 \mid \bm{x}_U, \bm{z}) = \mathbb{E}[\epsilon_1 \epsilon_2] = 0.
\end{equation}
The masked tokens become conditionally independent: $p(x_M^1, x_M^2 \mid \bm{x}_U, \bm{z}) = p(x_M^1 \mid \bm{x}_U, \bm{z}) p(x_M^2 \mid \bm{x}_U, \bm{z})$, and the conditional dependency gap vanishes: $\mathcal{T}(x_M^1, x_M^2 \mid \bm{x}_U, \bm{z}) = 0$.

\paragraph{Summary.}
By introducing $\bm{z}$, we reduce the residual correlation from $\sigma^2/(1+\sigma^2)$ to $0$. The latent variable ``explains away'' the common fluctuation caused by the unobserved global context, making the independence assumption valid for parallel generation.

\subsection{Generated Samples on OpenWebText}
\label{app:generated_samples}

We provide generated samples on OpenWebText at four sampling budgets.

\subsubsection{1024 Sampling Steps}
\label{fig:generated_samples_steps1024}
\begingroup\small
\texttt{<|im\_end|>} interest rates.

He said that overall, not including employment, they show lower interest rates than inflation in comparison to a higher level for older people.

While many people are indebted, many have been given a credit line unsupported on home ownership as the UK rose.

It is crucial to look at in real terms. In 2008 the pound hit back on the economy's growth, which was more than it was in 2007. Back in 2008 was \pounds 204,000 with the average adjusted pay rate of 62\%. It is important to bear to recognise, that much of the rise in the debt level was mainly due to a tightness of inflation, which came through deflation. However, as a result, we saw an increase in real value.

Inflation dragged the pound to a broader level.

He was forecasting the Federal Reserve Bank's strength in April. More than a month earlier, for the same period in April, the real debt level was about 20\% — approaching 27\% while the UK Government had experienced 28\%.

The implication, seems the government will raise debt levels less than a month earlier. However, while rates were only 23.04 in the UK on 19 July it reflects a reduction in mortgages for technical contracts, which it would expect between November and November.

An additional element of confidence level will start to be added to in new contracts, including requiring more purchases of a property to sell as an asset.

Mr Cherry said: "For the most recent year we experienced a rapid depreciation."

Stability of interest

He said that the higher rates are likely due to a slightly higher interest level, but leaves the possibility of employment over a wider period of the government's debt policy, without the potential for economic growth.

"Given the urgency level of the UK Government to start action on the other side to affect economic activity, given this context, I think there is a temporary start of more widespread economic activity."

The latest quarterly report suggests interest rates between 5.3\% and 4.1\%, though that the current "stable" funding status means, when interest rates continued to fall, the level of wages, which was large of the economy at this time and how much real pay kept falling.

This sector is subject to a high level of joblessness, with real interest rates relating to real asset holders.

Mr Cherry added that the UK economy has become almost certain, and with the coming boom with conditions hitting hard when it came to accelerating economic growth, the UK Government's overall impact on employment continued to be constrained by the nominal interest rate of 5 per cent, in 2011 and an estimated rate of wage increases in 1978, 1991, 1904 and 1978, to the current January levels for February 2015.

Real employment keeps falling

While the UK Government's nominal interest rate of 12 per cent fell in two years at 5.1 and down 4.8\%, there was a stand-off in the late 2000s when the biggest drag on employment was lower employment worldwide.

Yehy: "As compared with surveys of low employment, in large parts, that is understood that debt levels, or interest rates, steadily declined in 2007."

He said from the figures "the average real interest rate in 2014 was 7.4 and up 5 percent, indicating that most of us now have a higher level of employment."

He also said the higher balance sheets may have helped force students into spending more on student loans, thus increasing their income. Recent figures are particularly worrying as the UK ranks 36pc of students.

ADAPT

"A lot of young people are under this state who are losing employment - a number that is for the next parliament, I must say, yet.

"It has been devastating - who is living here with little hope or peace."

Yehy said that despite the overall rise in employment activity, between the first nine months last year the UK government was by 11 per cent, compared with 42 among four months over in a decade, then up by three or four years.

"This has been encouraging for lots of people in England and Ireland, in terms of employment, for the foreseeable future."

Yehy is concerned because of 27 unemployment extensions under which led to further job increases and surveys showing UK government alone experienced rising unemployment.

Unemployment is responsible for individual jobs-related economic losses as employment is a small part of people losing employment," he said.

"Sometimes when I hear people talk about unemployment extension cuts, they see - I can say, 'The cuts worked!"

'I don't think they're the best alternative to long-term unemployment'

"So I think people are a little unhappy. Some may sometimes have less savings than others.

"I think in terms of benefits, maybe it doesn't produce results, but I am told they will\texttt{<|im\_end|>}
\par\endgroup

\subsubsection{512 Sampling Steps}
\label{fig:generated_samples_steps512}
\begingroup\small
\texttt{<|im\_end|>}. He had one of the last systematic trips to the Rocky Mountains and Hollywood. Two years later, one of the first to be America’s reporter was to be on Arizona’s Ellulia Island, Hawaii. He brings us three stories from Camp Everest, when The Post introduced him as a reporter in January. He had no plans for hosting the program.

Frank Marousy, we arrived in La Mesa, on the Rocky Mountain, from the Grand Canyon. And I’m twenty years old and camped in La Mesa. In high school, I started to have a child, and got a job. And I thought I was fun.

Authorities wanted me to know your story. I wondered why you’re in the news when you had an interview with one former Indonesian journalist, the role of scientists on the changing carbon chain.

Simon found it hard to dig up a series of problems that she needed us to continue her activities as she told The Guardian on where her job and priorities came up at the time.

“Could you explain a little bit about the challenges that have been in the laboratory for years? Well, this is what the IPCC was focusing on. Energy, and they had identified these problems, were the only challenges,” Mr. Pauler asked two former editors of The Guardian and The Guardian.

We detailed on discussion of the available water resources on climate change. The Prime Minister reported to Malta and asked Mr. Simon as he followed up on climate negotiations, and asked what information he learned on the issues of climate change, he did ask about broader insights into seeing press releases on sites: “I have seen your interviews. has there was anything controversial?

“ Mr. Pauler said, “There was information about climate change. Facebook, we were doing the research. I’m an ecologist at Stanford University. I’m on tour. So I only got the information when I was overseas. Catching news drew me on there; which, consequently, allowed me to learn about climate changes more slowly and more widely. And I see the addresses of some very prominent working leaders and others about how much it has influenced her because around the globe it inspired me. I had to support the prime minister, to support the national crisis in Cambodia, ... to Beijing’s climate policies, Frank Marousy,” told The Guardian. We also spoke with Mr. Pauler about the new climate plan. They both agreed last week that they would hold meetings with Mr. Simon after a special session. Mr. Pauler taped a long video clip on Wednesday in the reporters’ office when he and fellow reporters discussed climate change on the scientists panel that is now organising his annual retreat.

Mr. Simon sat in a room with Mr. Pauler and she talked about addressing the financial difficulties focused on reducing emissions or helping their farmers defend themselves to frictions higher than their European ones. He also said, “It’s important to make sacrifices to be made from American food stocks, as well as our environment, from American food corporations — that they would protect themselves from food hazards and also have financial protection against food tariffs, because American farmers would be supported by their compensation for [security].”

Mr. Simon said, “To me, something that you say is consistent with my understanding of going to the energy policy being on how we should be reducing carbon pollution, she said. I think many of the kind of priorities that we need, there — we need to prioritize.”

He also added, “I would say we’re going to be so vulnerable because of the fact that there are already climate change and economic impacts out there mentioned. But that’s what my answer: Let’s make this country look better, so the health of our planet, not just the carbon-neutral environment but the health of people, and where the climate change is, that we need to start that change of climate change for everybody,” Mr. Simon told us.

“At global levels, when we added carbon emissions in the way that we’re pushing for than climate change, the science and energy policy is being de-recorded and used.

We have released some numbers for the last two years, Mr. Simon said. And the results of the new MGH report show that Australia is planning to reduce greenhouse gas emissions by a global level of 22.5 percent by 2050 from the 2010 current levels. They are also trying to limit emissions from the order of two percent in 2060 to zero by 2070, according to this aim of the UN report.

If the global economy was going to achieve the least 8.2 percent of 2050 in the amount of CO2 emissions to nations, doing so instead is immoral and immoral. I don’t think it should — even if it’s federal policy — ” he said. “That is impossible.” I like to think people say that we’re not getting to ourselves right now.

“The success of this international study is that there are very specific, measurable differences between these countries,” Mr. Simon said, referring to the provincial government carbon\texttt{<|im\_end|>}
\par\endgroup

\subsubsection{256 Sampling Steps}
\label{fig:generated_samples_steps256}
\begingroup\small
\texttt{<|im\_end|>} people who knew the TV footage provided by the deputies. About 20 police officers took the show of a water truck photographed, and saw television crews right the other side of the fire and fire cars.

One of it later came from Philadelphia fire departments.

“That would be a deputy chief officer, and then he got to his car,” said Philadelphia Police Sarah at the television station. Afterward, once he got closer, he looked at the “area of the Sheriff’s works, and it happened.”

“He said, What do you want to happen, John?”

“We have four trucks. We have five ones. I’m concerned about that because it’s different. And part of the plan was that they wanted the drivers to look at cars.”

Now with most firefighters, deputies responded, “You can’t identify any anybody.” Instead, local reporters tweeted out stories about a picture that sparked outrage. With local reporters like CNN and CBS broadcast wires to display their images depicting a Philadelphia flag as their own trucks, social media consolidated.

200 of the Philadelphia residents were angry that BPD officers were trying to plot the images into firefighters.

“There have been a lot of reactions from those guys who know that in these things, it seems more sinister,” said Abby Yates from Berryessa. Talking about it, “It’s convincing that they have something that they want because you’re there to see what’s being captured and that you will be right to go.”

On Facebook, Yates observed that there was a big quick reaction from the Philadelphia fire. It wasn’t a typical Philadelphia response.

The Telegraph reported some local concussion video snippets from the Harrisburg interview of the fire cars. The following link came from AT\&T:

“We you saw is that fire car customers have to pay for rent for immediate repairs,” Yates said.

Here is what the statement was: “If you see the same thing in a lot of a firefighter’s cars, would be some of them will have a life” to recease.

Cornell 71’s Thomas Miller said that he was one of the journalists in Philadelphia to work the panoramic video.

“All reporters are also interviewed by the Hall of Justice,” he said. “Look at technology and better, look at the cameras and then it’s different.” Miller also said that he was willing and determined to make such a video.

Jokes describing a toic reporter as well as a YouTube climber, Miller described to CNN, “it was really cool, this is really a way to get this video to happen to first responders of Division 5/6 of Philadelphia.”

"How did you make the video happen? Why this stuff?" Phillips asked. “Do you think there was responsible equipment with equipment and people,” he answered. "Do you support the America Foundation for Change and donate so much and have your support and help?"\texttt{<|im\_end|>}Most of the races got started in Emanuel's campaign for assistant Mayor Ped Pederson (where here are Philadelphia’s 18 editors) and last year was Eric Leverdehan, the owner of the newspaper, the owner of the Press and the owner of Wilson, who paid him huge amounts of money to help fund.

But in November, it started in earnest and went ahead. And it wasn't for single parent's children where the figure started growing in in the midly-1950s and '70s suburban side; it was the four other children of Debbie Leverdehan, who allegedly expressed a love for Ryan Schultz, in much of Sandy's warding.

The story happened shortly Debbie was 69 when her mother and her brother died recently in battling cancer from her family at the Center section of events. Sandy's aunt Kyle Schultz was visiting Rhode Island for a weekend before Sandy's dad died aged 61. Two of his sisters, Julia Owens, and Wayne Owens, had passed years from, and Lee, and Christine, Kelly. They first came here last Saturday from First Avenue in the Mapleland, and now four others in Cyslewood, Long Beach, when she was young five. And as soon as last was Debbie's trip, Debbie said of her son Sandy's getting to know she loved having Tim, Lee Schultz. She was a little beautiful. How was it? Debbie didn't even know how much. It was a heroic move because it was a beautiful moment.

"A big family had really cried for the mayor and Christine when she came on going. And this has happened from Kyle's dad to come the whole other way. In fact, it was actually on Debbie's end." Sandy's daughter, Johanna Schultz and the Julia Owens family, was again turned blue Friday by her father Dick, who predicted won again be running like her mother thought they were.

"I've got situations like this on here. There's a candidate that doesn't allow racial discrimination, I wouldn't say they'll get to this point," Debbie.

Other kids will be rallying against the plan, for the night of Sunday's quarter and a half season.

Photo: AP

Sunday 17-28\texttt{<|im\_end|>}
\par\endgroup

\subsubsection{128 Sampling Steps}
\label{fig:generated_samples_steps128}
\begingroup\small
\texttt{<|im\_end|>} We also expect about two or 10 volunteers who will introduce Bitcoin to our community.

As the first of those who write about Bitcoin there, the Canadian Bitcoin Foundation is an example — it is actually a bit of a silver roof to keep the money flowing in…

Bitcoin “root” represents an army of households that will feed them electricity and help us convert their income into urban and military — we will spread wealth all over the network.

Ensure that, check out a way to use Bitcoin.info and a blog post.

Create a Bitcoin cant repository and have a regular Bitcoin account.

Morgan’s Bricket Update

Before any new Bitcoin software is built in accordance with the initial materials we are going to start our for this purpose. Every Bitcoin node is measured on a character propagation and, eventually, we will happen a few things before it is all rolled on to the user. Bitcoin blocks is a fundamental specification with many types and multiple days., when you design the protocol, with this example, we would recommend that the Bitcoin block finish over time to do so. Then, that at the beginning of your BTC block development from Bitcoin Blocks is not interrupted by a 3rd fork. Bitcoin blocks can then be used as interrupters, even if the type has a format or to develop them as a way to ...read; the propagation of the Bitcoin blocks was not the way to handle transactions, and would be used by a miner. Based on what last in the above description, we still also recommend putting one or two spikes of a few thousand times throughout all of your BTC block. Try timed the “bittleneck” of the following block before the next transaction and change to 30 minutes a time if you want to get down to these transactions – 10 minutes – 30 as a BTC block. Design of these updates will be presented in the ends – use of the timestamp for Bitcoin Blocks for a broadcast at the end of the/block, resulting in a reverting process until the block is to complete in the preceding months.

We argue a lot against this structure, Bitcoin Bitcoinism, which mentioned that a certain type of Bitcoin enthusiast is likely to actively participate in Bitcoin developments. However, we believe these types of projects should be more of the same as before: to allow the community to understand the hard forks that may occur in the current Bitcoin code, and expect to be frequently replaced with more forks in the future.

It makes sense that the Bitcoin developers and the Bitcoin community would like to make our development easier, and we will try and not released them just so that it gets more sense. Before proceeding by saying that without a hard fork you may want to add another version or another version of Bitcoin, we caution you not to recognize the prominence of changes, it can take some time. The Bitcoin patchending has been updated and in the last 6 days, and our community will be working on this extension closely with the Bitcoin project. But as the fact of this, note of the state of the Bitcoin concept will help to change as to how Bitcoin should go about functions. For those who reach the Bitcoin community, this extension will help the Bitcoin community understand the future of Bitcoin, and their own future.

Although it will take into account details of the current standard of development Bitcoin Vision. In the Bitcoin community, as well as the Light Architecture’s development and the community expertise in Bitcoin Bitcoin development, we recognize that this is a high priority for the Bitcoin community, and hence, it is more important to provide support for others in our community. Therefore, as a result of the Light Architecture’s development, we organize the Bitcoin Community workshop for conversations about the standard of development Bitcoin Vision.

By the way,

The early developers of the Bitcoin Society also received informal feedback from a moderator on our subreddit, called btc, that preceded the discussions. Participants received various changes in their project, some changes were random, and so often came from the design and implementation of the current Bitcoin development software, the Bitcoin Developers posted a thread to explain how they proposal it rather than tell us that there are currently no added versions.

Coin Editions on the Internet

You are select the first set from the basic process to use the Bitcoin project: on the page of the Bitcoin development website, you will select the version to send back to you on the Bitcoin publication. This page contains relevant information from the Bitcoin Republic, with links to draft papers, papers which will be gathered in the format below: Select the current version of their foundation; for those who submitted their project, they can click you to download a PDF copy, and they will select the complete version of the foundation. If you downloading a copy agreed to pay a fee, strip the rest with it and then pay for their work.

Once the Bitcoin development is in the run, you will click on our website to download a stock picture, you will find one of the four sets: either one will the development project use on hand for the infrastructure, the other set for the distribution, you will also know a digital supply unit and you will add one other code to the \texttt{<|im\_end|>}
\par\endgroup

\end{document}